\documentclass{article}
\usepackage{lmodern}

\usepackage{fullpage}
\usepackage{cmbright}
\usepackage[utf8]{inputenc}
\usepackage[T1]{fontenc}
\usepackage[english]{babel}

\usepackage{graphicx,amsfonts,amsmath,amssymb,amsthm,bm}
\usepackage[authoryear,sort&compress]{natbib}

\usepackage{main-commands}
\usepackage{xurl}
\usepackage{longtable}

\usepackage{nicefrac}
\usepackage{cleveref}

\numberwithin{equation}{section}

\usepackage{tikz}
\usetikzlibrary{shapes.geometric,calc,positioning}

\newcommand{\la}{\lambda}

\newcommand{\btheta}{\bm\theta}

\usepackage{authblk}

\title{Generalization Bounds of Surrogate Policies\\ for Combinatorial Optimization Problems}

\date{\today}

\author[1]{{\small Pierre-Cyril} Aubin-Frankowski}
\author[1,2,3]{{\small Yohann} De Castro}
\author[1]{{\small Axel} Parmentier}
\author[5]{{\small Alessandro}~Rudi}

\affil[1]{\it \small CERMICS, CNRS, ENPC, Institut Polytechnique de Paris, Marne-la-Vallée, France.}
\affil[2]{\it \small Institut Camille Jordan, École Centrale Lyon, CNRS UMR 5208, France.}
\affil[3]{\it \small Institut Universitaire de France (IUF)}
\affil[5]{\it \small SIERRA, INRIA Paris, France. }

\begin{document}

\maketitle

\begin{abstract}
Many real-world decision problems require solving, again and again, combinatorial optimization instances drawn from a common distribution. A recent line of structured learning methods exploits this regularity by learning policies that pair a statistical model with a tractable combinatorial oracle, instead of solving each instance independently. Training such policies is notoriously difficult, however: the resulting empirical risk is piecewise constant in the model parameters, which hinders gradient-based optimization, and only a few theoretical guarantees have been provided so far. We address this issue by analyzing smoothed (perturbed) policies: adding controlled random perturbations to the direction used by the linear oracle yields a differentiable surrogate risk and improves generalization. Our main contribution is a generalization bound that decomposes the excess risk into $(\mathit{i})$ perturbation bias, $(\mathit{ii})$ statistical estimation error, and $(\mathit{iii})$ optimization error. The perturbation bias is controlled by the \emph{fan-crossing probability}, a new geometric quantity measuring the likelihood that a perturbation changes the oracle solution. We introduce two complementary conditions to bound it~--- the \emph{Uniformly Bounded Density} (UBD) property, yielding a sharp $\mathcal{O}(\lambda)$ bias, and the weaker \emph{Uniform Weak moment} (UW) property, yielding a sub-linear bound~--- both capturing the geometric interaction between the statistical model and the normal fan of the feasible polytope. The statistical estimation error is controlled via a uniform deviation bound over the policy class, with rate $\mathcal{O}(1/(\lambda\sqrt{n}))$ that scales inversely in the smoothing parameter. Concerning the optimization error, we exploit kernel Sum-of-Squares methods to mitigate the curse of dimensionality of global optimization. We illustrate the scope of the results on applications such as stochastic vehicle scheduling, highlighting how smoothing enables both tractable training and controlled generalization. The framework is in particular applicable to contextual stochastic optimization.
\end{abstract}

\section{Introduction}

\subsection{Learning policies instead of minimizing separately}
Consider the following combinatorial optimization problem
\begin{equation}
    \label{eq:hardProblem}
    \min_{\vy \in \calY(\vx)}\fh(\vy,\vx)\,
\end{equation}
where $\calY(\vx)$ is the finite set of feasible solutions for instance $\vx\in\calX$, {\it e.g.,}\ a path $\vy$ in a weighted graph $\vx$. The $\R$-valued objective $f^0$ may model any complex phenomena, and is not restricted to any specific class of functions. We assume however that $f^0$ comes with an oracle, that can be mildly expensive to compute, {\it i.e.}, can take a few seconds to evaluate and for instance be the result of a simulation. We will focus on settings where the problem defined by $f^0$ is challenging in two ways: (i) its modelling is too complex to be cast as a mathematical programming formulation that can be solved with off-the-shelf algorithms; (ii) its computational cost prohibits the use of a local descent heuristic.

Instead of solving each instance independently, we consider a family $\calH$ of policies $h : \vx \in \calX \mapsto \vy\in \calY(\vx)$. Our goal is then to {\it learn} a mapping~$h^\star$ that minimizes the risk
\begin{equation}
    \label{eq:risk}
    \min_{h\in \calH} \calR(h)  \coloneqq \bbE_{X}\big[ \fh(h(X), X) \big],
\end{equation}
where $X$ is a random instance distributed according to some distribution $\bbP_X$. If the set $\calH$ corresponds to all the possible mappings $\calX\to (\calY(\vx))_{\vx\in \calX}$, then we get back to the classical approach where each instance is solved independently by  considering $\vx\mapsto \argmin_{\vy \in \calY(\vx)}\fh(\vy,\vx)$.

We instead restrict ourselves to a family of tractable policies parametrized by a vector $\bfw$ belonging to a set~$\calW$, and we learn the parameter $\bfw$ that minimizes the risk. This approach has proved empirically successful in the last few years on a variety of applications such as portfolio optimization~\citep{elmachtoubSmartPredictThen2021}, path finding in images~\citep{poganvcic2019differentiation}, vehicle scheduling~\citep{parmentierLearningApproximateIndustrial2021}, route planning~\citep{Ferber2023}, machine scheduling~\citep{parmentierLearningSolveSingle2021}, dynamic vehicle routing~\citep{batyCombinatorialOptimizationEnrichedMachine2024}. For such an approach to work, the policy mappings of $\calH$ need to be able to properly explore the large combinatorial set of solutions $\calY(\vx)$. The approach amounts to delegating the exploration of~$\calY(\vx)$ to the linear optimization problem
\begin{equation}
    \label{eq:ylinearProblem}
    \hat \vy_\vx(\btheta)\in\argmax_{\vy \in \calY(\vx)}\, \langle \vy , \btheta  \rangle\,,
\end{equation}
where $\hat \vy_\vx\,:\,\btheta\in\bbR^{d(\vx)} \mapsto \hat \vy_\vx(\btheta)\in\calY(\vx)$ is an oracle returning an optimal solution of~\eqref{eq:ylinearProblem}, vector~$\btheta$ is a direction in $\bbR^{d(\vx)}$ and $\langle \cdot,\cdot\rangle$ denotes the standard inner product. We assume in this article to have a practically efficient algorithm to solve \eqref{eq:ylinearProblem}. Our oracle for fixed $\btheta$ is for instance a LP solver, a MILP solver, a shortest path algorithm,\dots The statistical model $\featw$ is then used to find a direction vector $\btheta$ that leads to solutions $\vy$ with low $\fh(\vy,\vx)$. In other words, we consider the following family of parametrized policy mappings
\begin{equation}
    \label{eq:modelClass}
    \calH_\calW \coloneqq
    \big \{h_\vw\,:\,\vx\in\calX \mapsto \hat \vy_\vx(\featw(\vx))\in\calY(\vx)\,,\ \vw\in\calW \big\}\,,
\end{equation}
for some parameter space $\calW \subset \R^{d_\calW}$. Note that the objective of~\eqref{eq:ylinearProblem} nevertheless relies on an embedding of~$\calY(\vx)$ in $\bbR^{d(\vx)}$, whose dimension $d(\vx)$ may depend on the instance considered, and that different $\btheta$, and hence~$\vw$, can lead to the same solution $\hat \vy_\vx(\btheta)$. Figure~\ref{fig:stoVSP_policy} illustrates such a policy on a stochastic vehicle scheduling problem.

\begin{figure}[!ht]
    \centering
\begin{tikzpicture}
\def\l{0.6cm}

\node[align=center,draw] (sm) {Statistical model \\ $\psi_{\vw}$};
\node[align=center,draw,right=2.5 of sm] (co) {Comb.~optimization \\ (minimum cost flow) \\ $\displaystyle\max_{\vy \in \calY(\vx)}\langle \vtheta,\vy\rangle$};

\draw[<-] (sm) to node[midway, below] {Instance} node[midway,above] (x) {$\vx$}  ++(-4,0) node[left] {$h_\vw:$};
\draw[->] (sm) to node[midway, below] {Arc weights} node[midway,above] (theta) {$\vtheta = (\theta_a)_{a \in A}$} (co);
\draw[->] (co) to node[midway,below] {solution} node[midway,above] (y) {$\vy$} ++(5,0);

\node[above = 0.4 of x] (solution) {
    \begin{tikzpicture}[scale=0.6, transform shape]
        \node[draw, text width=1cm, align=center] (o)  at (1*\l,1) {$o$};
        \node[draw, text width=1cm, align=center] (d)  at (11*\l,1) {$d$};
        \node[draw,circle] (v1) at (2*\l,4) {$v_1$};
        \node[draw,circle] (v2) at (3.5*\l,1) {$v_2$};
        \node[draw,circle] (v3) at (4*\l,2) {$v_3$};
        \node[draw,circle] (v4) at (7*\l,4) {$v_4$};
        \node[draw,circle] (v5) at (7*\l,2.5) {$v_5$};
        \node[draw,circle] (v6) at (7*\l,1) {$v_6$};
        \node[draw,circle] (v7) at (9*\l,2) {$v_7$};
        \node[draw,circle] (v8) at (9*\l,4) {$v_8$};

        \draw[->,dashed] (3*\l,5) to node[midway]{$\phi(a,\vx)$} (8*\l,5);

        \draw[->,dashed] (v1) --  (v3);
        \draw[->,dashed] (v2) -- (v3);
        \draw[->,dashed] (v3) -- (v4);
        \draw[->,dashed] (v3) -- (v6);
        \draw[->,dashed] (v5) -- (v8);
        \draw[->,dashed] (v6) -- (v7);
        \draw[->,dashed] (o) -- (v2);
        \draw[->,dashed] (v2) to (v6);
        \draw[->,dashed] (v6) -- (d);
        \draw[->,dashed] (o) -- (v1);
        \draw[->,dashed] (v1) -- (v4);
        \draw[->,dashed] (v4) -- (v8);
        \draw[->,dashed] (v8) -- (d);
        \draw[->,dashed] (o) -- (v3);
        \draw[->,dashed] (v3) -- (v5);
        \draw[->,dashed] (v5) -- (v7);
        \draw[->] (v7) -- (d);

    \end{tikzpicture}
};

\node[above = 0.4 of theta] (solution) {
    \begin{tikzpicture}[scale=0.6, transform shape]
        \node[draw, text width=1cm, align=center] (o)  at (1*\l,1) {$o$};
        \node[draw, text width=1cm, align=center] (d)  at (11*\l,1) {$d$};
        \node[draw,circle] (v1) at (2*\l,4) {$v_1$};
        \node[draw,circle] (v2) at (3.5*\l,1) {$v_2$};
        \node[draw,circle] (v3) at (4*\l,2) {$v_3$};
        \node[draw,circle] (v4) at (7*\l,4) {$v_4$};
        \node[draw,circle] (v5) at (7*\l,2.5) {$v_5$};
        \node[draw,circle] (v6) at (7*\l,1) {$v_6$};
        \node[draw,circle] (v7) at (9*\l,2) {$v_7$};
        \node[draw,circle] (v8) at (9*\l,4) {$v_8$};

        \draw[->,dotted] (3*\l,5) to node[midway]{$\theta_a$} (8*\l,5);

        \draw[->,dotted] (v1) --  (v3);
        \draw[->,dotted] (v2) -- (v3);
        \draw[->,dotted] (v3) -- (v4);
        \draw[->,dotted] (v3) -- (v6);
        \draw[->,dotted] (v5) -- (v8);
        \draw[->,dotted] (v6) -- (v7);
        \draw[->,dotted] (o) -- (v2);
        \draw[->,dotted] (v2) to (v6);
        \draw[->,dotted] (v6) -- (d);
        \draw[->,dotted] (o) -- (v1);
        \draw[->,dotted] (v1) -- (v4);
        \draw[->,dotted] (v4) -- (v8);
        \draw[->,dotted] (v8) -- (d);
        \draw[->,dotted] (o) -- (v3);
        \draw[->,dotted] (v3) -- (v5);
        \draw[->,dotted] (v5) -- (v7);
        \draw[->] (v7) -- (d);
    \end{tikzpicture}
};

\node[above = 0.4 of y] (solution) {
    \begin{tikzpicture}[scale=0.6, transform shape]
        \node[draw, text width=1cm, align=center] (o)  at (1*\l,1) {$o$};
        \node[draw, text width=1cm, align=center] (d)  at (11*\l,1) {$d$};
        \node[draw,circle] (v1) at (2*\l,4) {$v_1$};
        \node[draw,circle] (v2) at (3.5*\l,1) {$v_2$};
        \node[draw,circle] (v3) at (4*\l,2) {$v_3$};
        \node[draw,circle] (v4) at (7*\l,4) {$v_4$};
        \node[draw,circle] (v5) at (7*\l,2.5) {$v_5$};
        \node[draw,circle] (v6) at (7*\l,1) {$v_6$};
        \node[draw,circle] (v7) at (9*\l,2) {$v_7$};
        \node[draw,circle] (v8) at (9*\l,4) {$v_8$};

        \draw[->,thick,color=purple] (o) -- (v2);
        \draw[->,thick,color=purple] (v2) to (v6);
        \draw[->,thick,color=purple] (v6) -- (d);

        \draw[->,thick,color=red] (o) -- (v1);
        \draw[->,thick,color=red] (v1) -- (v4);
        \draw[->,thick,color=red] (v4) -- (v8);
        \draw[->,thick,color=red] (v8) -- (d);

        \draw[->,thick,color=blue] (o) -- (v3);
        \draw[->,thick,color=blue] (v3) -- (v5);
        \draw[->,thick,color=blue] (v5) -- (v7);
        \draw[->,thick,color=blue] (v7) -- (d);

    \end{tikzpicture}
};
\end{tikzpicture}
    \caption{Illustration of the stochastic vehicle scheduling policy.
    }
    \label{fig:stoVSP_policy}
\end{figure}

As classical in machine learning, we then consider an empirical version $\calR_n$ of the risk $\calR$ in \eqref{eq:risk} which is approximated through samples $(X_1,\ldots,X_n)$ by
\begin{equation}
    \label{eq:regretMinimization}
    \min_{\vw \in \calW} \calR_n(h_\vw),
    \quad \text{where} \quad
    \calR_n(h_\vw) = \frac1n\sum_{i=1}^n f^0\big(h_\vw(X_i),X_i\big).
\end{equation}
Unfortunately, due to the combinatorial nature of~\eqref{eq:ylinearProblem}, this objective is piecewise constant in $\vw$, which has two major issues: it makes the optimization problem intractable, and leads to poor generalization properties.
Consequently, we rely also on a perturbation strategy developed in~\citet{berthetLearningDifferentiablePerturbed2020} to turn the linear optimization problem~\eqref{eq:ylinearProblem} into a family of tractable distributions $p_\lambda(\vy|\vtheta)$, with $\lambda>0$ a parameter controlling the smoothness of $\vtheta \mapsto p_\lambda(\vy|\vtheta)$. We then formulate the learning problem as
\begin{equation}
    \label{eq:def_pbm_reg}
    \min_{\vw \in \calW} \calR_{n,\lambda}(h_\vw) ,
    \quad \text{where} \quad
    \calR_{n,\lambda}(h_\vw) = \frac1n\sum_{i=1}^n \bbE_{\vy \sim p_\la(\vy|\psi_{\vw}(X_i))}\big[f^0\big(\vy,X_i\big)\big],
\end{equation}
where $p_\la$ is defined through
\begin{equation*}
    p_\la(\vy|\btheta)\coloneqq \mathbb P_Z\big[\vy \in \argmax_{\vy'\in\calY(\vx)} \langle \vy',\btheta+\lambda Z(\vx) \rangle \big],
\end{equation*}
with $Z:\calX\to \bbR^{d(\vx)}$ an isotropic random variable, {\it e.g.,}\ Gaussian. For $\lambda=0$, through $p_\la$ we recover  \eqref{eq:ylinearProblem}, and \eqref{eq:def_pbm_reg} coincides with \eqref{eq:regretMinimization}. We refer to the more formal Definitions~\ref{def:surrogate_proba} and \ref{def:perturbed_proba} for the handling of non-unique optimal~$\vy'$. In practice, we consider a Monte-Carlo approximation of the expectation in \eqref{eq:def_pbm_reg} by sampling over~$Z$~\citep{berthetLearningDifferentiablePerturbed2020, parmentierLearningStructuredApproximations2021}. Through the regularization, we obtain a smooth objective~\eqref{eq:def_pbm_reg} which is more amenable for optimization, whose error we will also have to control. In this paper, we choose to study the so-called kernel Sum-of-Squares method \citep{rudiFindingGlobalMinima2020} since it is among the few methods providing a bound on the optimization error for which the smoothness of the objective function $\calR_{n,\lambda}$ can mitigate the curse of dimensionality in the dimension $d_\calW$ of $\calW$.

\subsection{Discussion on the methodological choices}

\paragraph{Why consider a learning problem?}
The idea of transforming an optimization problem into a learning one is motivated by industrial practice.
In many real-world scenarios, the problem instances come from a specific but unknown distribution.
For example, an airline optimizes the routes of its airplanes based on a schedule and some operational constraints that do not vary much from one day to another. Learning approaches balance the computational burden between \emph{learning}, when the policy is trained offline on historical data, and there is typically no industrial process deadline implying a limited time budget, and \emph{inference}, when the calibrated policy is used on-the-fly to find a solution for a given instance, and a strong deadline must typically be respected for industrial settings.
While in classical approaches directly optimizing $\fh(\cdot,\vx)$, all the computation time occurs during inference, our approach makes most of the computation occur during learning: only the linear surrogate~\eqref{eq:ylinearProblem} needs to be solved  during inference. This is helpful in industrial settings, where the full information on the instance is often known at the last minute, which strongly constrains the computing time and forces one to use only objective functions simple enough, such as \eqref{eq:ylinearProblem}, to allow for fast optimization algorithms.

\paragraph{Why not approximate $f^0$ directly?}
Since combinatorial algorithms must scale to be useful in OR, practitioners generally simplify~\eqref{eq:hardProblem} into
\begin{equation}\label{eq:over_simplified}
    \min_{\vy \in \calY(\vx)}\tilde \fh(\vy,\vx)\,
\end{equation}
where $\tilde \fh$ is a possibly-over-simplified deterministic objective amenable to classical algorithms. Our oracle \eqref{eq:ylinearProblem} does fall within this category, however our crucial variable is $\vw$, to coordinate between instances. Moreover we do not look to approximate the values of $\fh$ by $\tilde \fh$, and we do not discard $\fh$ since it is crucial in \eqref{eq:def_pbm_reg}. We do wish nevertheless for $\hat \vy_\vx(\btheta)$ to be close to $\argmin \fh(\cdot,\vx)$, in terms of values of $\fh$. The learning approaches focus on optimizing the true objective~$\fh$, while the estimate-then-optimize approach forgets about $\fh$ and considers only its approximation $\tilde f^0$ during the optimization phase of~\eqref{eq:over_simplified}, and may overfit $\tilde \fh$.

\paragraph{Why consider a linear oracle in \eqref{eq:ylinearProblem}?}
The choice of the oracle~\eqref{eq:ylinearProblem} amounts to an architecture choice for a neural network. Indeed the oracle appears as a final composition in \eqref{eq:modelClass} (see \Cref{fig:stoVSP_policy}), hence we consider it as a combinatorial layer in our pipeline (see \Cref{fig:pipeline} below). This layer is application-dependent and critical for performance. A good surrogate model should be simple enough to have efficient algorithms, while being expressive enough to adapt to many settings. We consider a linear oracle because it provides such a balance.

Linear oracles are the de facto choice in optimization-augmented learning because the framework only needs a linear minimization oracle over the polytope $\calC = \conv(\calY)$, leaving the actual solver free. By Meyer's theorem~\citep{Meyer1974}, the convex hull of feasible solutions of any bounded MILP is itself a polytope, so any MILP-formulable problem fits as a linear layer. Additionally, our approach only requires linearity in $\vtheta$: for a layer $\max_{\vz} \vtheta^\top \phi(\vz)$ with non-linear mapping $\phi$, we can define $\vy = \phi(\vz)$ to return to our setting.

\paragraph{Why not assume access to a training set?} Finding a suitable policy in $\calH$ requires designing a learning algorithm that selects a parameter~$\vw$ that leads to a good performance in practice.
Two typologies of learning strategies have been contrasted in~\citep{bengioMachineLearningCombinatorial2021}: supervised learning, which is also referred to as imitation learning or inverse optimization depending on the community, and risk minimization,  a.k.a.~learning by experience. In risk minimization as in \eqref{eq:regretMinimization}, one does not have samples of near-optimal solutions. By contrast supervised learning approaches assume that the training set contains instance solution pairs $(\vx_1, \bar\vy_1),\ldots,(\vx_n,\bar\vy_n)$ of~\eqref{eq:hardProblem}. In this setting, the original objective~$f^0$ is dismissed, and surrogate losses $ \ell_{imitation}$ evaluate how far the model output~$\hat \vy_\vx(\featw(\vx_i))\in\calY(\vx_i)$ is from the training set target~$\bar\vy_i$. Nevertheless this setting still fits our framework by setting $\tilde \vx=(\vx, \bar \vy)$, assuming the existence of a joint law on $(X,\bar Y)$, and setting $f^0(\tilde \vx_i,\vy)\coloneqq  \ell_{imitation}(\vx_i, \bar \vy_i, \vy) $.

\paragraph{Why focus on the risk on $\calH_\calW$ rather than on any measurable $h$?} Our excess risk bounds are stated with respect to the best policy in the parametric class $\calH_\calW$ (cf.~\eqref{eq:modelClass}) instead of the Bayes optimal policy over all measurable mappings. This restriction is deliberate for several, complementary reasons. For any (possibly unattainable) measurable $h$, the classical risk decomposition writes
    \[
        \calR(\hat h) - \inf_{h\,\text{meas.}} \calR(h)
        = \underbrace{\big(\inf_{h\in \calH_\calW} \calR(h) - \inf_{h\,\text{meas.}} \calR(h)\big)}_{\text{approximation / misspecification}}
        + \underbrace{\big(\calR(\hat h)-\inf_{h\in \calH_\calW}\calR(h)\big)}_{\text{estimation + optimization + smoothing (this paper)}}.
    \]
    The first (misspecification) term depends intricately on the chosen architecture (depth, width, feature design, encoding of $\calY(\vx)$, etc.) and the combinatorial framework encapsulated in $\fh$. This term is largely orthogonal to the \emph{geometric} phenomena (normal fan structure and perturbation smoothing) we study in this paper. By focusing on $\calH_\calW$, we isolate the three components we can control: statistical estimation, optimization of~$\calR_{n,\lambda}$, and the perturbation bias under only a minimal assumption on $\fh$ (boundedness). Handling the first (approximation) term lies outside the scope of this paper: it is architecture- and application-dependent and would require additional modelling of the instance distribution together with regularity assumptions on (near-)optimal solution paths.

\paragraph{Are there other frameworks for bounding the optimization error?}  To the best of our knowledge, many global optimization algorithms converge because they provably do an exhaustive search, as is the case for DIRECT for instance \cite[Section 5]{Jones1993}. Consequently, their convergence bounds are most often affected by the curse of dimensionality. Here we have an advantage since $\calR_{n,\la}$ is a smooth function thanks to the parameter $\la>0$. This incites us to use a specific kernel-based approach recently suggested in \cite{rudiFindingGlobalMinima2020} to leverage the smoothness of the function $\calR_{n,\la}$ in our theoretical bounds. In practice however, we recommend using SGD to leverage existing solvers.

Concerning SGD or other gradient-based solvers, the perturbation $p_\la$ was introduced in \cite[Definition 2.1]{berthetLearningDifferentiablePerturbed2020} to compute the needed gradient.
\cite[Proposition 4.1]{dalle2022learning} provides the formula for the gradient $\nabla_{\bftheta}$ of the objective of~\eqref{eq:def_pbm_reg}, and $\nabla_{\bfw}$ is then deduced by the chain rule.
However this is a score-based estimation of the gradient, which is known to have a high variance.
We refer to \cite{dalle2022learning} for more insights on differentiating through combinatorial layers.

\medskip

In a nutshell, instead of approximating the objective function~$\fh$ in \eqref{eq:hardProblem} by a simpler~$\tilde \fh$ as in \eqref{eq:over_simplified}, we study an average error $\bbE_{X}\big[ \fh(h(X), X) \big]$ as in~\eqref{eq:risk}. The latter involves a statistical model~$\calH$ defined in~\eqref{eq:modelClass} and derived from the linear combinatorial optimization problem~\eqref{eq:ylinearProblem}. The role of the model is to explore efficiently the solution set $\calY(\vx)$ and  we do not seek to approximate~$\fh$ in any way. Our Theorem~\ref{thm:A3_statistical} provides error guarantees when $n$ goes to $+\infty$ to show the convergence of problem \eqref{eq:regretMinimization} to~\eqref{eq:risk}. Note that convergence is towards the best policy in the hypothesis class $\calH_\calW$ defined in~\eqref{eq:modelClass}, which may be far from the optimal Bayes policy if $\calH$ has been poorly chosen.

\subsection{Related works}
\label{sec:related_works}

\paragraph{Relevant applications}
Parametrized policies $h_{\vw}$ defined in Equation~\eqref{eq:modelClass} have demonstrated state-of-the-art empirical performance across several families of combinatorial decision problems. We briefly outline key application patterns while keeping the focus on why a \emph{policy-learning with linear oracle} viewpoint is attractive. In \emph{black-box optimization}, evaluating $\fh(\vy,\vx)$ requires an expensive simulation ({\it e.g.,} delay propagation, stochastic congestion), and directly solving \eqref{eq:hardProblem} is prohibitive when $|\calY(\vx)|$ is large. Rather than constructing a global simplified objective $\tilde f^0$, we learn a direction $\psi_{\vw}(\vx)$ that is passed to a tractable linear oracle, using it as an \emph{exploration mechanism} of $\calY(\vx)$ rather than as a pointwise approximation of $\fh$, {\it e.g.,} \cite{vu2017surrogate}. This also allows for \emph{scaling and fast heuristics} as, even if $\fh$ is explicit, exact solvers or MILP formulations may fail to scale. Learning $\vw$ shifts instead the heavy computation to offline: one learns $h_\vw$ on moderate size instances (offline) and infers large size instance solutions using one linear oracle call. This yields high-quality fast heuristics in large-scale scheduling \citep{parmentierLearningSolveSingle2021} and bi-level contexts \citep{ferrarini2025learning}. Another example is \emph{stochastic optimization} where $\fh$ embeds expectations over uncertain parameters ({\it e.g.,} stochastic programming \citep{shapiro2021lectures}), and traditional decomposition handles only moderate sizes. Recent studies on surrogate policy learning have shown strong performance in stochastic vehicle scheduling \citep{parmentierLearningApproximateIndustrial2021}, two-stage structured approximations \citep{dalle2022learning,parmentierLearningStructuredApproximations2021}, and routing \citep{ahmed2024districtnet}. Last but not least, for \emph{contextual and multi-stage variants} of stochastic optimization \citep{sadana2024survey} (see Example~\ref{ex:contextual_optimization} below), the side information replaces explicit conditional expectation estimation, and surrogate-policy extensions towards multi-stage or dynamic routing and equilibrium-inspired flow models appear in challenge-winning or emerging work \citep{batyCombinatorialOptimizationEnrichedMachine2024,jungel2022structured,EUROMeetsNeurIPS,jungel2025learning,greif2024combinatorial,jungel2025wardropnet}.

The present work provides geometric and probabilistic principles (normal fan structure + perturbation smoothing) explaining these outcomes.

\medskip

To illustrate our theoretical guarantees, we present three examples, one generic and two more industrial.
\begin{example}(Contextual Stochastic Optimization)
    \label{ex:contextual_optimization}
    Consider a stochastic optimization problem where the cost over a solution is influenced by some noise $\xi$. The decision maker does not observe~$\xi$, but has access to some contextual data $\tilde X$ that is correlated to~$\xi$.
    A solution is a policy $h$ that maps any context $\tilde \bfx$ to a feasible decision $\bfy \in \calY(\tilde \bfx)$ given $\tilde \bfx$.
    Given a joint distribution over $(\tilde X,\xi)$ and a hypothesis class $\calH$ for $h$, the goal of contextual stochastic optimization is to find the policy that minimizes the expected risk
    \begin{equation}
    \notag
        \min_{h \in \calH} \bbE\Big[f^\rmc\big(h(\tilde X),\tilde X,\xi\big)\Big],
    \end{equation}
     where $f^c(\bfy,\tilde \vx, \xi)$ is the cost incurred if decision $\bfy$ is taken under context $\tilde \vx$ and uncertainty $\xi$. An optimal policy is provided by the conditional distribution
    \begin{equation}\label{eq:optContextualStoPolicy}
        h^\star : \tilde \bfx \longmapsto \argmin_{\vy \in \calY(\tilde \vx)} \bbE\Big[f^\rmc\big(\vy,\tilde X,\xi\big)\Big| \tilde X = \tilde\vx\Big].
    \end{equation}
    In contextual stochastic optimization however, see \citet{sadana2024survey} for a survey of this emergent field, the joint distribution over~$(\tilde X,\xi)$ is unknown and we only have access to a training set $(\tilde \vx_1, \bfxi_1),\ldots,(\tilde \vx_n,\bfxi_n)$.
    This makes the optimal policy~\eqref{eq:optContextualStoPolicy} impractical.

    Contextual stochastic optimization can be considered as a special case of our setting. It suffices to define~$X$ as $(\tilde X,\xi)$, and $\fh(\bfy,\bfx)$ as $f^\rmc(\bfy,\tilde\bfx,\bfxi)$ for $\bfx = (\tilde \bfx,\bfxi)$. We then restrict ourselves to models~$\featw(\bfx)$ of the form~$\tilde \psi_{\bfw}(\tilde \bfx)$ that exploit only the contextual information and not the noise. Similarly we can restrict the policies $h$ not to depend on $\bfxi$ to preserve the fact that $\bfxi$ is uncertain at decision, {\it i.e.}\ $h(\tilde \bfx,\bfxi)= \tilde h(\tilde \bfx)$ for some $\tilde h$. Common practice in stochastic optimization would have suggested defining $X$ as~$\tilde X$ and
    $\fh(\bfy,\bfx)$ as the objective of \eqref{eq:optContextualStoPolicy}. However this would require the conditional distribution of $\xi$ given $\tilde X = \vx_i$, which we do not know.
\end{example}

\begin{example}\label{ex:stoVSP}
    (Stochastic Vehicle Scheduling \citep{parmentierLearningApproximateIndustrial2021})
    Consider a set of tasks with fixed begin and end times operated by a fleet of vehicles. When a task is late, the assigned vehicle may propagate the delay to the next tasks. The goal is to build the sequence of tasks operated by each vehicle in such a way that each task is operated by exactly one vehicle and the expected cost of delay~$\fh(\vy,\vx)$ is minimized.
    The instance~$\vx$ can typically be encoded as a digraph~$D = (V,A)$ where~$V$ is the set of tasks, and there is an arc~$a$ in $A$ between~$v$ and~$v'$ if these two tasks can be operated by the same vehicle in a sequence, and a vector~$\bfphi(a,\vx)$ of additional information for each arc~$a$. Since a solution corresponds to a partition of the vertices of an acyclic digraph into paths, the set of feasible solutions~$\calY(\vx)$ can be identified with the vertices of the flow polytope over~$D$. As a statistical model, one can take~$\psi_{\bfw}: \bfx \mapsto \bftheta = (\theta_a)_{a \in A}$ with~$\theta_a = \langle \bfw, \bfphi(a,\vx) \rangle$ for each~$a$.
    The corresponding policy is illustrated on Figure~\ref{fig:stoVSP_policy}. Aircraft routing problems solved routinely by airlines are variants of this problem. In particular, the stochastic version aims at controlling delay propagation. In practice, solvers used in production typically address the linear oracle. Consequently, switching from the deterministic and delay-agnostic approaches currently in use to our policy that avoids delay propagation does not alter the inference algorithm in production; it only requires a change in the parametrization of the inference solver. While computing this new parametrization involves solving a computationally intensive learning problem, this step is carried out during the offline learning rather than during on-the-fly inference.
\end{example}

\begin{example}\label{ex:scheduling}
(Single Machine Scheduling \citep{parmentierLearningSolveSingle2021})
In single machine scheduling problems, $n$ jobs have to be performed on a given machine. A solution is a permutation of $\{1,\ldots,n\}$ giving the order in which the jobs are processed. Several objectives can be used. In the reference above, jobs have release times, before which they cannot be started, and processing times, which indicate the time needed by the machine to operate the job. The vector $\bfy$ encodes a permutation, and the goal here is to minimize $\fh(\bfy,\vx)$ which is equal to the sum (the average) of the job completion times, taking into account the release time.
Permutations $\bfy$ can be encoded by vectors $\bfy_i$ giving the position of job $i$ in the permutation. Optimizing a linear objective $\langle \bftheta,\bfy\rangle$ amounts to sorting $\bftheta$, and we can therefore use a sorting algorithm as a linear minimization oracle.
\end{example}

\paragraph{Learning algorithms and generalization bounds.}

A prominent approach combining linear surrogates~\eqref{eq:ylinearProblem} and risk minimization is ``Smart Predict then Optimize'' (SPO)~\citep{elmachtoubSmartPredictThen2021}.
SPO replaces the discontinuous regret loss with a convex surrogate (SPO+), leveraging the duality of linear programming to convexify the decision boundary.
However, generalization guarantees in this framework~\citep{ElBalghiti2023} rely heavily on the linearity of the objective and fixed feasible sets $\calY$ independent of $\vx$.
Consequently, the duality arguments underpinning SPO+ collapse when applied to generic, potentially black-box non-linear functions $f^0(\vy, \vx)$ with instance-dependent decision boundaries, as there is no natural convex surrogate available.\\
In contrast, we investigate the risk minimization paradigm~\eqref{eq:def_pbm_reg} where the loss is a smoothed version of the objective.
This approach has received limited attention as discussed in~\citep{parmentierLearningStructuredApproximations2021}, with exceptions found in \citet{qi2021integrated}, \citet{Ferber2023}, and recent works on contextual stochastic optimization~\citep{bouvier2025primal, hoppe2025structured}.
While we share the SPO goal of minimizing downstream decision regret rather than prediction error, we address the intractability of the empirical risk~\eqref{eq:regretMinimization} via \textit{randomized smoothing}~\citep{berthetLearningDifferentiablePerturbed2020} rather than convex relaxation.\\
This methodological divergence necessitates a distinct theoretical analysis.
We cannot rely on margin-based bounds such as the ``distance to degeneracy'' employed by \citet{ElBalghiti2019gen}, as we do not bound the complexity of a discrete loss directly.
Instead, we must control the \textit{smoothing bias} $|\mathcal{R}_\lambda - \mathcal{R}_{\varepsilon_0}|$.
To this end, we introduce the \emph{Uniform Weak (UW) property} (Section~\ref{sec:setting_contributions}) and employ probabilistic assumptions (Assumption~\ref{A_law}) rather than Lipschitz continuity of the loss.
This allows us to bound the approximation error and ensure that minimizing the smoothed surrogate yields a performant solution for the true combinatorial problem.\\
Finally, our problem relates to data-driven algorithm design~\citep{gupta2016pac, balcan2020chapter}.
While \citet{balcanHowMuchData2021} provide guarantees for piecewise decomposable risks, their results do not apply to our smoothed, regularized risk.
Our Theorem~\ref{thm:A3_statistical} provides the counterpart to their results, leveraging tools tailored to our smoothed non-linear context.

\subsection{Contributions}
In this paper, we build upon the normal fan geometry (see Section~\ref{sec:surrogate_policy}) associated with the linear optimization program~\eqref{eq:ylinearProblem} and upon a perturbation approach to deliver a generalization bound (Theorem~\ref{thm:principal}) that controls the perturbation bias, the statistical learning error, and the optimization error.

Perturbation makes the learning problem smooth and amenable to gradient descent methods. Furthermore, we prove non-asymptotic guarantees on the performance of the learned policy, as well as convergence rates towards the best policy encoded by~\eqref{eq:modelClass}.

To obtain our performance guarantees, we introduce the \emph{fan-crossing probability} $q_\vw(\lambda)$ (see~\eqref{eq:def_Vw}), which measures the probability that a perturbation of scale $\lambda$ changes the oracle solution. This quantity controls the perturbation bias. To bound it, we introduce two complementary conditions: the uniform weak (UW) moment property, which yields a sub-linear bound on the bias, and the uniformly bounded density (UBD) property, which yields a sharp linear bound. Both properties quantify the interplay of the statistical model and the normal fan of~\eqref{eq:ylinearProblem}. We show that UW holds under mild assumptions on the statistical model $\psi_\vw$ and on $\bbP_X$, and we illustrate this on several examples in Section~\ref{sec:UW_examples}.

\subsection{Outline and notation}
The next section (Section~\ref{sec:surrogate_policy}) contains our surrogate policy model (Section~\ref{subsec:surrogate_policy}), the different kinds of risks (Section~\ref{subsec:risks}) and the main theorem on generalization guarantees (Section~\ref{subsec:guarantees}). Section~\ref{sec:setting_contributions} is devoted to the proofs of the main theorems. It starts with the working assumptions (Section~\ref{sec:conditions}), then we introduce the fan-crossing probability $q_\vw(\lambda)$ and derive the perturbation bias bounds under the Uniform Weak (UW) and Uniformly Bounded Density (UBD) properties (Section~\ref{sec:pertub_bias}). Section~\ref{sec:ERM} focuses on the discrepancy between the risk and its empirical version. Finally, we provide in Section~\ref{sec:k-SoS_control} the optimization error bound when using kernel Sum-of-Squares.

\medskip

We denote by $\mathcal L(U)$ the law of a random variable $U$ and by $\mathbb P_U$ (resp.\ $\mathbb E_U$) the probability (resp.\ expectation) with respect to the law $\mathcal L(U)$. We denote the stochastic boundedness by $\mathcal O_\bbP$: for a sequence of random variables~$(U_n)_{n}$, notation $U_n=\mathcal O_\bbP(1)$ means that for any $\varepsilon>0$, there exists a finite $M>0$ and a finite $m\geq 0$ such that for every $n\geq m$, $\mathbb P(|U_n|>M)\leq\varepsilon$. For a sequence of non-zero reals $(a_n)_n$, $U_n=\mathcal O_\bbP(a_n)$ denotes that $(U_n/a_n)=\mathcal O_\bbP(1)$.

\medskip

The set of instances is denoted by $\calX$, and the set of feasible solutions for an instance $\vx\in\calX$ by~$\calY(\vx)\subseteq \bbR^{d(\vx)}$. Given an instance $\vx\in\calX$, the objective function is~$\fh(\cdot,\vx)$ whose optimal solution is denoted by $\vy^0(\vx)$ or simply~$\vy^0$. The oscillation of $\fh$ is defined by
\begin{equation}
    \label{eq:def_osc}
    \operatorname{osc}(\fh)\coloneqq\sup_{\vy,\vx}\big\{\fh(\vy,\vx)\big\}
    - \inf_{\vy,\vx} \big\{\fh(\vy,\vx)\big\}\,.
\end{equation}
The set of parameters is denoted by $\calW$ and the statistical model by $\featw\,:\,\vx\in\calX\mapsto\btheta\in \bbR^{d(\vx)} $. The solution associated with parameter $\vw\in\calW$ is $\hat\vy_\vx(\featw(\vx))$.

\medskip

\noindent\textbf{Main notation table.} The core symbols used throughout the paper are listed in the appendix (Table~\ref{tab:notations_intro}); when a formal definition appears in the text we cite the corresponding reference.

Standard symbols ($\bbR, \bbP, \bbE$, norms, inner product) follow conventional usage. Subscript $n$ denotes empirical versions; $\lambda$ indicates the perturbation that yields a smooth (differentiable) risk.

\newpage

\section{The surrogate policy model and its guarantees}
\label{sec:surrogate_policy}

\subsection{Surrogate policy model}
\label{subsec:surrogate_policy}

\begin{figure}[!ht]
    \begin{center}
        \begin{tikzpicture}
            \node (o) {};
            \node[draw,right=2cm of o,align=center] (m) {Statistical \\ model $\featw(\vx)$};
            \node[draw,right=3cm of m,align=center] (co) {CO oracle $\hat \vy_\vx$ \\  $\displaystyle\max_{\vy \in \calY(\vx)} \langle \vy , \vtheta \rangle$};
            \node[right=3cm of co,draw, align=center] (d) {Objective \\ $\fh(\vy)$};

            \draw[->] (o) to node[above]{Instance} node[below]{$\vx \in \calX$} (m);
            \draw[->] (m) to node[above]{Direction vector} node[below]{$\bftheta \in \bbR^{d(\vx)}$} (co);
            \draw[->] (co) to node[above]{Surrogate Policy} node[below]{$\vy \in \calY(\vx)$} (d);
            \draw[dashed,->] (d.south) -- ([yshift=-0.5cm]d.south)
                 -- node[midway, below]{Learning on $\vw \in \calW$} ([yshift=-0.5cm]m.south)
                 -- (m.south);
        \end{tikzpicture}
    \end{center}
    \caption{Surrogate policy encoded by the statistical model $\featw\,:\,\vx\in\calX \mapsto \btheta\in\bbR^{d(\vx)}$ with combinatorial optimization (CO) layer given by a linear program over solutions $\vy\in\calY(\vx)$.}
    \label{fig:pipeline}
\end{figure}

Figure~\ref{fig:pipeline} illustrates the policies we consider.
We start with a statistical model $\vx\mapsto\featw(\vx)\in\bbR^{d(\vx)}$ that embeds an instance $\vx\in\calX$ into~$\mathds R^{d(\vx)}$ and we set $\btheta\coloneqq\featw(\vx)\in\bbR^{d(\vx)}$. We build our surrogate policy $\hat \vy_\vx(\featw(\vx))\in\calY(\vx)$ by considering a solution to the linear program
\begin{equation}
    \label{eq:ylinearProblem_sec2}
     \max_{\vy \in \calC(\vx)} \langle \vy , \btheta  \rangle \text{ where }\btheta\coloneqq\featw(\vx).
\end{equation}
We maximize over the convex hull
\[
    \calC(\vx)\coloneqq\conv(\calY(\vx)),
\]
which is a polytope since $\calY(\vx)$ is finite. We assume to have a tractable oracle for this problem.
It is given by the family of surrogate maps $(\vy,\vx;\vw) \mapsto \langle \vy , \featw(\vx)  \rangle$ linear in $\btheta=\featw(\vx)$.

\medskip

By construction, the surrogate policy $\hat \vy_\vx(\btheta)$ is an extreme point of $\calC(\vx)$. For each $\vy\in\calY(\vx)$ one can define the normal cone $\calF_{\vy}$ as the set of directions $\btheta$ such that $\vy$ is a solution to \eqref{eq:ylinearProblem_sec2}, equivalently
\[
    \calF_{\vy}\coloneqq\Big\{\btheta\in\bbR^{d(\vx)}\ \Big|\ \langle \vy,\btheta\rangle \geq \langle \vy',\btheta\rangle \text{ for all } \vy'\in \calC(\vx) \Big\}.
\]
The normal fan $\mathcal{N}(\calC(\vx))$ is the collection of all such cones, as illustrated in Figure~\ref{fig:normal_cone}.

\begin{figure}[!h]
    \centering
    \begin{tabular}{rcl}
    \begin{tikzpicture}[scale=2.5]

\coordinate (A) at (0.3,0);
\coordinate (B) at (1,0);
\coordinate (C) at (1.2,0.6);
\coordinate (D) at (0.8,1);
\coordinate (E) at (0.2,1);
\coordinate (F) at (-0.2,0.6);

\draw (A) -- (B) -- (C) -- (D) -- (E) -- (F) -- cycle;

\node[below left] at (A) {$\vy_5$};
\node[below right] at (B) {$\vy_6$};
\node[above] at (C) {$\vy_1$};
\node[above right] at (D) {$\vy_2$};
\node[above left] at (E) {$\vy_3$};
\node[above left] at (F) {$\vy_4$};

\coordinate (Center) at (C);

\draw[dashed,fill=gray!20] (Center) -- ($(Center)!0.75cm!95:(B)$) arc (0:45:1cm) -- cycle;

\draw[->,thick] (Center) -- ($(Center)!0.5cm!110:(B)$);

\node[right] at ($(Center)!0.5cm!110:(B)$) {$\btheta$};

\end{tikzpicture}
& &

\begin{tikzpicture}

\coordinate (O) at (0,0);

\foreach \angle in {-10, 40, 90, 130, 220, 270} {
    \draw (O) -- (\angle:2.5);
}
\fill[gray!30] (O) -- (-10:2.5) arc (-10:40:2.5) -- cycle;

\node at (20:2) {$\calF_{\vy_1}$};
\node at (65:2) {$\calF_{\vy_2}$};
\node at (110:2) {$\calF_{\vy_3}$};
\node at (180:2) {$\calF_{\vy_4}$};
\node at (245:2) {$\calF_{\vy_5}$};
\node at (-50:2) {$\calF_{\vy_6}$};

\draw[->,thick] (O) -- (3:1.5);
\draw (3:1.5) circle (0.34);
\draw[-] (3:1.5) -- (-9.75:1.468);
\draw (15:0.9) node{$\btheta$};
\draw (-3.5:1.57) node{\tiny $\rho$};

\end{tikzpicture}
\end{tabular}
    \caption{Normal cone at point $\vy_1$ to the polytope (left) and normal fan with internal radius $\rho$ at point $\btheta$ (right).}
    \label{fig:normal_cone}
\end{figure}
\noindent
Normal cones are not disjoint, but their interiors are. We break boundary ties arbitrarily.

\begin{continuance}{ex:stoVSP}
    In the context of stochastic vehicle scheduling, $\calY(\bfx)$ corresponds to the partitions of the vertices $V$ into paths. Solutions are encoded by indicator vectors $\bfy = (y_a)_{a \in A}$ where $y_a$ is a binary variable equal to $1$ if arc $a$ is in the solution, and to $0$ otherwise. Its convex hull, $\calC(\bfx)$, is known to be a flow polytope and has a compact formulation.
\end{continuance}

\begin{continuance}{ex:scheduling}
    The feasible solutions $\calY(\bfx)$ of single machine scheduling are the vectors~$\bfy$ encoding a permutation of $\{1,\ldots,n\}$. Their convex hull $\calC(\bfx)$ is a polytope known as the permutahedron.
\end{continuance}

\medskip

For a given $\vx\in\calX$, it can occur that several $\vy\in\cal Y$ are optimal, {\it i.e.,}~there exist several $\vy\in\calY(\vx)$ such that~$\featw(\vx)\in\calF_\vy$ and $\featw(\vx)$ lies on the boundary of those cones. To cover this case, we represent $\hat \vy_\vx(\featw(\vx))$ as a probability distribution on such $\vy$'s. A formal definition of the surrogate policy $\hat \vy_\vx(\featw(\vx))$ is given as follows.

\begin{de}[Surrogate policy]\label{def:surrogate_proba}
Let $\btheta\in\mathbb R^{d(\vx)}$, and define the internal cone radius as
\begin{equation}\label{eq:def_rho}
        \rho(\btheta)
        \coloneqq\sup\Big\{r>0\ \Big|\ \forall \vu\in B(0,1)\,, \exists \vy\in\calY(\vx)\text{ s.t. } \btheta\in\calF_\vy\text{ and }\btheta+r \vu \in\calF_\vy\Big\}\,,
\end{equation}
where $B(0,1)\subset \mathbb R^{d(\vx)}$ is the standard Euclidean ball. We define the surrogate policy measure
\[
    \hat \vy_\vx(\btheta) \coloneqq \sum_{\vy\in\calY(\vx)} p_0(\vy|\btheta)\delta_\vy
\]
where $\delta_\vy$ is the Dirac mass at point $\vy$ and
\begin{equation}\label{eq:def_p_0}   p_0(\vy|\btheta)\coloneqq \mathbb P_U\big[\btheta+\rho(\btheta)U\in\calF_\vy\big]
\end{equation}
where $U$ is uniformly distributed on $B(0,1)$.
\end{de}

\noindent
Note that $p_0(\vy|\btheta)$ measures the proportion of the normal cone $\calF_\vy$ locally around $\btheta$. The quantity $\rho(\btheta)$ will have an important role in our context of perturbations $\mathbf{z}$, since $\btheta+\mathbf{z}$ can only point to a $\vy\neq \hat \vy_\vx(\btheta)$ if $\|\mathbf{z}\|_2\ge \rho(\btheta)$.  Perturbations too small do not affect the outcome.

\begin{remark}[Generic case and an abuse of notation]
Generically, a direction $\btheta$ is in the interior of a normal cone~$\calF_\vy$ and~$\vy'\mapsto p_0(\vy'|\btheta)$ is zero except at point $\vy$. In this case, the measure $\hat \vy_\vx(\btheta)$ is a Dirac mass at point~$\vy$ and, by an abuse of notation, we denote by $\hat \vy_\vx(\btheta)$ the vector $\vy\in\bbR^{d(\vx)}$. To ease notation, we will consider that, by convention,
\[
f(\hat \vy_\vx(\btheta))\coloneqq\sum_{\vy\in\calY(\vx)} p_0(\vy|\btheta) f(\vy)
\]
for all functions $f$ and all $\btheta\in\bbR^{d(\vx)}$.
\end{remark}

\subsection{The regularized and empirical risks}
\label{subsec:risks}
In our algorithms, we optimize an empirical and regularized version of the risk $\calR$, which we now introduce.
We observe $n$ instances $X_1,\ldots,X_n$ drawn independently according to a (possibly unknown) distribution $\mathds P_X$. For all~$\vw\in\calW$, we denote the policy by $h_{\vw}\,:\,\vx \mapsto \hat \vy_\vx(\psi_{\vw}(\vx))$ and introduce the regularized population risk $\calR_\varepsilon$ and its empirical version $\calR_{n,\varepsilon}$ by
\begin{align}
    \calR_\varepsilon(h_\vw)    &\coloneqq\bbE_{X,Z}\big[
                            \fh(\hat \vy_\vx(\featw(X)+\varepsilon Z(X)), X)
                            \big]
                            \,,\label{eq:reg_risk_solution}\\
    \calR_{n,\varepsilon}(h_\vw)&\coloneqq\frac1n\sum_{i=1}^n\mathds E_Z\Big\{\big[
                            \fh(\hat \vy_\vx(\featw(X_i)+\varepsilon Z(X_i)), X_i)
                            \big]\Big\}
                            \,,\label{eq:emp_reg_risk_solution}
\end{align}
where $\varepsilon\geq 0$ is some regularization parameter and $Z(\cdot)$ is a perturbation defined as follows.

\begin{de}[Law of the perturbation]
    Given $\vx\in\calX$, $Z(\vx)$ has the same law as $R\times U$ where the random vector~$U$ is uniformly distributed on the unit Euclidean sphere of~$\bbR^{d(\vx)}$ and the random positive scalar $R\in(0,\infty)$ is independent of $U$.
\end{de}

This perturbation is inspired by the one used in~\citet{berthetLearningDifferentiablePerturbed2020, parmentierLearningStructuredApproximations2021} to define a differentiable loss function for structured prediction. The uniform distribution on the sphere ensures that all directions are equally likely, while the random scaling $R$ allows one to control the tail of the distribution of $Z(\vx)$.

\begin{remark}[On forthcoming Gaussian assumption]
 When explicit constants are needed in bounds involving $Z$, we adopt the assumption that $R$ is \emph{chi-distributed}, which implies that $Z$ is Gaussian (Assumption~\ref{A3} of the next section); all the resulting calculations are presented in Appendix~\ref{app:technical_lemmas}. The main body maintains the abstract notation $(R,U)$ to preserve generality across perturbation laws.
\end{remark}
\begin{remark}[Link with the internal radius]
    The perturbed vector $\featw(\vx)+\varepsilon Z(\vx)$ belongs to the normal cone of $\featw(\vx)$ whenever $\varepsilon R\leq \rho(\featw(\vx))$, where~$\rho(\featw(\vx))$ is the internal radius of the normal cone at $\featw(\vx)$. Note that this latter condition does not depend on the direction $U(G)$ of the perturbation~$Z(\vx)$.
\end{remark}

\medskip

\noindent For later use in our analysis we introduce the following quantity.
\begin{de}[Perturbed surrogate policy probabilities]\label{def:perturbed_proba}
Let $\lambda>0$ and $\vx\in\calX$. Let~$\btheta\in\bbR^{d(\vx)}$ and $\vy\in\calY(\vx)$. We define the perturbed surrogate policy probability at solution point $\vy$ for the direction~$\btheta$ as
\begin{equation}\label{eq:def_p_lambda}   p_\la(\vy|\btheta)\coloneqq\bbE_Z[p_0(\vy|\btheta+\lambda Z(\vx))]\,.
\end{equation}
\end{de}
\noindent
Note that $p_\la(\vy|\btheta)$ is the probability that $\vy$ is selected by the surrogate policy when the direction $\btheta$ is perturbed by $\lambda Z(\vx)$, where $Z(\vx)$ is defined above.

\medskip

The (non-perturbed) \emph{surrogate policy} of $\vx\in\calX$ is defined as a solution $\hat \vy_\vx(\featw(\vx))\in\bbR^{d(\vx)}$ to~\eqref{eq:ylinearProblem}, and the {surrogate policy model} is defined by~\eqref{eq:modelClass}.

\medskip

Now, let ${\varepsilon_0}\geq 0$ and consider the risk
\begin{align}
\label{eq:optimal_risk}
    \calR_{\varepsilon_0}(\calH)
    &\coloneqq\min_{\vw\in \calW}
    \big\{\bbE_{X,Z}\big[
        \fh(\hat \vy_\vx(\featw(X)+{\varepsilon_0} Z(X)), X)
    \big]\big\}\,,\\
    &\stackrel{\eqref{eq:def_p_0}}{=}\min_{\vw\in \calW}\notag
    \big\{\bbE_{X,Z}\big[
        \sum_{\vy\in\calY(\vx)} p_0(\vy|\featw(X)+{\varepsilon_0} Z(X))\, \fh(\vy,X)
    \big]\big\}\,, \\
    &\stackrel{\eqref{eq:def_p_lambda}}{=}\min_{\vw\in \calW}\notag
    \big\{\bbE_{X}\big[
        \sum_{\vy\in\calY(\vx)} p_{\varepsilon_0}(\vy|\featw(X))\, \fh(\vy,X)
    \big]\big\}\,, \\
    &= \inf_{\vw\in \calW}\calR_{\varepsilon_0}(h_\vw)\notag\,,\\
    &=\calR_{\varepsilon_0}(h_{\vw^\star})\,.\notag
\end{align}
We aim at finding the optimal surrogate policy $h_{\vw^\star}\,:\,\vx \mapsto \hat \vy_\vx(\psi_{\vw^\star}(\vx))$, where $\vw^\star$ is a solution to~\eqref{eq:optimal_risk}, assumed to exist.

\begin{remark}(cases $\varepsilon_0=0$ and $\varepsilon_0>0$)
    In the case where $\varepsilon_0=0$, note that $\calR_0(h_{\vw})=\bbE_{X}\big[\fh(\hat \vy_\vx(\featw(X)), X) \big]$ and $\vw\mapsto \calR_0(h_{\vw})$ may fail to be continuous in $\vw$. Hence it is not guaranteed that a solution to~\eqref{eq:optimal_risk} exists. This justifies considering a small regularization $\varepsilon_0>0$. It should be understood as a minimal amount of regularization to introduce, playing the role of machine precision.

    In the case $\varepsilon_0>0$, we will prove that the risk is continuous and hence a solution always exists, under mild assumptions (Proposition~\ref{pro:sol_OK}). Furthermore, we will also prove that the Property~\ref{A_law}, introduced later, holds, under mild assumptions (Proposition~\ref{prop:UW_conv}). So rather than studying a perturbation scaling $\lambda$ such that $\lambda\in(0,+\infty)$, we suggest instead to fix $\varepsilon_0>0$ and focus on $\lambda\in [\varepsilon_0,+\infty)$.
\end{remark}

\subsection{Guarantees}
\label{subsec:guarantees}

\paragraph{Policy model misspecification error.}
We define the {policy model misspecification error} as
\begin{align*}
    \mathcal E_{\varepsilon_0}(\calH)
        &\coloneqq \calR_{\varepsilon_0}(\calH) - \bbE_{X}\big[ \fh(\vy^0(X), X) \big] \\
        &= \bbE_{X,Z}\big[\fh(\hat \vy_\vx(\psi_{\vw^\star}(X)+{\varepsilon_0}Z(X)), X) \big]
                        - \bbE_{X}\big[ \fh(\vy^0(X), X) \big]\geq 0
\end{align*}
where $\vy^0(\vx)$ minimizes \eqref{eq:hardProblem}.

\medskip

In $\bbP_X$-average, the quantity $\mathcal E_{\varepsilon_0}$ measures how far the policy $\hat \vy_\vx\big(\psi_{\vw^\star}(X)+\varepsilon_0 Z(X)\big)$ is from matching the quality of $\vy^0(X)$ for the loss $\fh$. Equivalently, it quantifies how often the direction $\psi_{\vw^\star}(X)+\varepsilon_0 Z(X)$ fails to fall inside normal cones $\mathcal F_{\vy}$ corresponding to solutions $\vy$ with small excess loss $\fh(\vy,X)-\fh(\vy^0(X),X)$. Controlling the term $\mathcal E_{\varepsilon_0}(\calH)$ would require additional regularity assumptions on $\vy^0$ and $\fh$, which lies beyond the scope of this paper; we therefore focus on the optimal risk in \eqref{eq:optimal_risk}. In practice, choosing an appropriate architecture for $\psi_\vw$ is what keeps the model misspecification error small. Empirical evidence from the literature (see Section~\ref{sec:related_works}) shows that, in many applications, the learned policy attains a small expected optimality gap, indicating that the policy model misspecification error is limited.

\medskip

The model misspecification error can sometimes be \emph{explicitly} bounded. If an approximation (or heuristic) algorithm $\mathsf A$ for \eqref{eq:hardProblem} with ratio $\alpha\ge 1$ can be encoded as a policy $h_{\vw_{\mathsf A}}\in\calH_\calW$, then
\[
    \calR_{\varepsilon_0}(\calH) \le \alpha\, \bbE_X[\fh(\vy^0(X),X)] + \varepsilon_0\Delta_{\text{pert.}}\quad\Rightarrow\quad
    \mathcal E_{\varepsilon_0}(\calH) \le (\alpha-1)\,\bbE_X[\fh(\vy^0(X),X)] + \varepsilon_0\Delta_{\text{pert.}},
\]
where $\Delta_{\text{pert.}}$ is a constant which may depend on the law $Z$ (perturbation). Hence a certified ratio inside $\calH_\calW$ yields a constructive upper bound; choosing to enrich the features or the architecture can only tighten this bound. We refer to \cite[Section 4.3]{parmentierLearningStructuredApproximations2021} for this structured approximation viewpoint on the maximum weight two-stage spanning tree problem and the objective function approximation framework.

\medskip

We focus on the \emph{geometric} phenomena (normal fan structure and perturbation smoothing) by isolating the three components we can control: statistical estimation, optimization of $\calR_{n,\lambda}$, and the perturbation bias.

\paragraph{Error bounds.} We will introduce an estimator $h_{\vw_{M,n,\la}}$ which depends on some complexity integer parameter $M\geq 1$, the sample size $n$, and a real valued tuning parameter $\la> 0$ such that $\la\geq\varepsilon_0$. Assume that $\vw^\star$ exists, then one has the following decomposition of the error, with each term handled separately,
\begin{align*}
0\leq  \; &\calR_{\varepsilon_0}(h_{\vw_{M,n,\la}})- \bbE_{X}\big[\fh(\vy^0(X), X)\big]\\
    &=
    \underbrace{\calR_{\varepsilon_0}(h_{\vw_{M,n,\la}})-\calR_{\la}(h_{\vw_{M,n,\la}})}_{\text{Theorem~\ref{thm:reg_excess}}}
    +\underbrace{\calR_{\la}(h_{\vw_{M,n,\la}})-\calR_{n,\la}(h_{\vw_{M,n,\la}})}_{\text{Theorem~\ref{thm:A3_statistical}}}
    \\
    &
    +\underbrace{\calR_{n,\la}(h_{\vw_{M,n,\la}})-\calR_{n,\la}(h_{\vw_{n,\la}})}_{\text{Theorem~\ref{thm:k-SoS_certificate} }}
    +
    \underbrace{\calR_{n,\la}(h_{\vw_{n,\la}})-\calR_{n,\la}(h_{\vw^\star})}_{\leq 0}
    \\
    &
    +
    \underbrace{\calR_{n,\la}(h_{\vw^\star})-\calR_{\la}(h_{\vw^\star})}_{\text{Theorem~\ref{thm:A3_statistical}}}
    +
    \underbrace{\calR_{\la}(h_{\vw^\star})-\calR_{\varepsilon_0}(h_{\vw^\star})}_{\text{Theorem~\ref{thm:reg_excess}}}
    \\
    &
    +
    \mathcal E_{\varepsilon_0}(\calH)
\end{align*}
where, for any $\vw\in\calW$, we define $h_{\vw_{n,\la}}$ such that
\[
\calR_{n,\la}(h_{\vw_{n,\la}})=\min_{\vw\in\calW}\calR_{n,\la}(h_{\vw})\,.
\]
In practice, we will make some working assumptions, to obtain the following, possibly probabilistic, bounds on each of the terms.
\begin{theo}
\label{thm:principal}
    Under the conditions given in Section~\ref{sec:conditions}, the following holds true. Let $\varepsilon_0\geq 0$ and $\lambda>0$ be such that $\lambda\geq \varepsilon_0$. Let $\tau\in(0,1)$. There exists a constant $C>0$ that depends only on $\varepsilon_0$, $\tau$ and $\fh$ such that for any $\vw\in\calW$ and $n\geq 1$, one has
\begin{align*}
|\calR_{\varepsilon_0}(h_{\vw})-\calR_{\la}(h_{\vw})|
    &\leq C \la^\tau\mathrm{polylog}(\la)
        &\text{(Theorem~\ref{thm:reg_excess}, Perturbation bias)}\\
|\calR_{\la}(h_{\vw})-\calR_{n,\la}(h_{\vw})|
    &=\mathcal O_{\mathds P}\Big(\frac1{\la\sqrt n}\Big)
&\text{(Theorem~\ref{thm:A3_statistical}, Empirical process)}\\
\intertext{where $\mathrm{polylog}(\la)$ is a polynomial logarithm term and, for $h_{\vw_{M,n,\la}}$ given by the kernel Sum-of-Squares estimate solution to $\displaystyle\min_{\vw\in\calW}\calR_{n,\la}$, }
|\calR_{n,\la}(h_{\vw_{M,n,\la}})-\calR_{n,\la}(h_{\vw_{n,\la}})|
&=\mathcal O_{\mathds P}\Bigg[
\bigg(
\frac{1}
{\la(M/\log M)^{\frac{1}{d_\calW}}}\bigg)^{s-\frac{d_\calW}{2}}
\Bigg]
&\text{by Theorem~\ref{thm:k-SoS_certificate} (k-SoS)}
\end{align*}
where $s> {d_\calW}/2$ is some tuning parameter on the order of regularity of the admissible functions and $M$ a number of queries within k-SoS. These estimates yield the following bound on the error between the risk and the regularized risk
\[
0\leq\calR_{\varepsilon_0}(h_{\vw_{M,n,\la}})- \bbE_{X}\big[\fh(\vy^0(X), X)\big]
\leq C \la^\tau\mathrm{polylog}(\la)
+\mathcal O_{\mathds P}\Bigg[
\Big(\frac1{\la\sqrt n}\Big)+\bigg(
\frac{1}
{\la(M/\log M)^{\frac{1}{d_\calW}}}\bigg)^{s-\frac{d_\calW}{2}}
\Bigg]
+\mathcal E_{\varepsilon_0}(\calH)\,,
\]
where the k-SoS bound is given up to a small a posteriori error, see Remark~\ref{rem:k-SoS}.
\end{theo}

\paragraph{Improved performances via the Linear Perturbation Bias}
As stated in Theorem~\ref{thm:principal}, the perturbation bias scales sub-linearly as $\mathcal{O}(\la^\tau\mathrm{polylog}(\la))$. This fractional rate stems from the generality of the Uniform Weak~(UW) property $($see Section~\ref{sec:setting_contributions}$\,)$, which restricts the exponent to $\tau \in (0,1)$ to accommodate complex distributions. \\
\indent
However, this rate can be strictly improved to a \emph{sharp linear bound} of $\mathcal{O}(\la)$ under slightly more regular geometric conditions (sharpness is proven in Proposition~\ref{prop:sharpness_lower_bound}). If the statistical model pushes forward the instance distribution such that the learned directions admit a \emph{uniformly bounded probability density on a compact support}, then the probability mass falling near the normal fan boundaries scales linearly with the perturbation $\la$. Under these conditions, the excess risk bound established in Theorem~\ref{thm:principal} improves to, for all $\lambda>0$,
\begin{equation}
\label{eq:linear_bias_excess}
    0 \leq \calR_{\la}(h_{\vw_{M,n,\la}}) - \bbE_{X}\big[\fh(\vy^0(X), X)\big] \leq C \la + \mathcal{O}_{\mathds P}\Bigg[ \frac{1}{\la\sqrt{n}} + \bigg( \frac{1}{\la(M/\log M)^{\frac{1}{d_\calW}}} \bigg)^{s-\frac{d_\calW}{2}} \Bigg] + \mathcal{E}_{0}(\calH) \,.
\end{equation}
where $\mathcal{E}_{0}(\calH)=\inf_{\vw\in \calW}\calR_{0}(h_\vw)-\bbE_{X}\big[ \fh(\vy^0(X), X) \big]\geq 0$ and the k-SoS bound is given again up to a small a posteriori error. We formalize this exact geometric behavior and prove the sharper $\mathcal{O}(\la)$ perturbation bias later in Proposition~\ref{prop:linear_bias_bound} (Section~\ref{sec:pertub_bias}) and Theorem~\ref{thm:linear_excess_risk}, which is a specification of Theorem~\ref{thm:principal} under Assumption~\ref{A_UBD}.
\paragraph{Perturbation level guidance.} Our excess risk bound is of the type {\it statistical to computational}, showing the effect of regularization $\la$, sampling $n$, and the complexity of the algorithm $M$, here the k-SoS algorithm. It shows that, when $\la$ goes to zero, the excess risk goes to zero if $\sqrt{n}$ dominates $1/\la$ and $M$ dominates $1/\la^{d_\calW}$. The latter is a so-called curse of dimensionality, the complexity of the k-SoS algorithm being exponential in the dimension ${d_\calW}$, while $1/\la$ can then be interpreted as a smoothness parameter. Indeed, \emph{the domination $1/\la$ by $\sqrt{n}$ shows quantitatively how many sample points the practitioner needs when diminishing the regularization~$\la$}. The optimization error is for its part based on the number of sampled points~$M$ on $\calW$ where the objective $\calR_{n,\la}$ is evaluated, $s$ corresponding to which Sobolev space of smoothness $s$ the function $\calR_{n,\la}$ belongs to, see \Cref{sec:k-SoS_control}. In practice, for a Gaussian smoothing, $s \gg d_\calW$ and could in principle be taken infinite. However some constants hidden in the $O_{\mathds P}$ can depend exponentially on~$s$, so we cannot take $s\to \infty$. This bound on the optimization error is to be assessed for its theoretical value of mitigating the curse of dimensionality. Algorithms based on k-SoS rely on SDP programming, limiting $M$ to a few hundreds. We refer to \cite[Section 10]{rudiFindingGlobalMinima2020} for numerical benchmarks.

\section{Main results}
\label{sec:setting_contributions}

We present in this section the results we use to obtain Theorem~\ref{thm:principal}. As discussed above, our analysis involves a parameter $\varepsilon_0\geq 0$ that accounts for a small perturbation defining the risk $\calR_{\varepsilon_0}$. After presenting the working assumptions (Section~\ref{sec:conditions}), we introduce the \emph{fan-crossing probability}~$q_\vw(\lambda)$, a pivot quantity that controls the perturbation bias (Section~\ref{sec:pertub_bias}). To bound~$q_\vw(\lambda)$, we introduce two complementary conditions: the Uniformly Bounded Density (UBD) property, which yields a sharp linear bound, and the weaker Uniform Weak (UW) moment property, which yields a sub-linear bound. We show that UW is always satisfied when $\varepsilon_0>0$, and give four cases in which it holds when $\varepsilon_0=0$.

\subsection{The standard conditions}
\label{sec:conditions}

\paragraph{Partition of the instances.}
On many combinatorial optimization problems, instances $\vx$ are pairs~$(G,\xi)$, where~$G$ is a combinatorial object, and $\xi_G$ is a vector of parameters whose dimensions depend only on $G$, {\it e.g.,}~$G$ is typically a graph and $\xi_G$ contains a vector of parameters (a label) for each vertex or edge of $G$.
Furthermore, we assume to know an embedding of $\calY(\vx)$ in the Euclidean space~$\bbR^{d(\vx)}$ with $d(\vx)$ of polynomial size in $\vx$.
For instance, solutions of the shortest path problems are vertices of the path polytope, solutions of the minimum weight spanning tree problem are vertices of the spanning tree polytope, and so on.
More generally, most combinatorial optimization problems admit a mixed integer linear programming formulation.
As illustrated by these examples, the set of feasible solutions $\calY(\vx)$ frequently depends only on $G$. To avoid technicalities, we omit $G$ and $\xi_G$ in the analysis but make the following assumption about a partition of the space of instances.

\begin{assum}{Part} The following holds.
\label{A_emb}
{\ }
\begin{itemize}
    \item There is a finite partition $\mathcal G$ of $\calX$ into $(\calX_G)_{G\in\mathcal G}$ such that, for each $G \in \calG$, the sets $\calY(\vx)$ are the same and finite for all $\vx\in\calX_G$;
    \item The absolute value of the target function $|\fh|$ is uniformly bounded on $\calX_G$ for each $G\in\calG$ and, hence, as $\mathcal G$ is finite, its oscillation is finite, $\operatorname{osc}(\fh)<\infty$.
\end{itemize}
\end{assum}

    Under Assumption~\ref{A_emb}, for each $G\in\calG$, and all $\vx\in\calX_G$, the embedding of~$\calY(\vx)$ into~$\mathds R^{d(\vx)}$ and the dimension $\vx\mapsto d(\vx)$ are constant. We denote by $d(G)$ the value of $d(\vx)$ for $\vx\in\calX_G$. Without loss of generality, we assume that $\calY(\vx)\subseteq\mathds R^{d(\vx)}$ and $\mathrm{Vect}(\calY(\vx))=\mathds R^{d(\vx)}$.

\begin{continuance}{ex:stoVSP}
    In the context of stochastic vehicle scheduling, $G$ corresponds to the digraph~$D$ we introduced on p.~\pageref{ex:stoVSP}, and $\calG$ to all possible digraphs. Assuming that $\calG$ is finite is therefore natural: since no firm is going to serve an infinite number of requests, $\calG$ is a subset of the set of acyclic digraphs with at most a fixed number of vertices. Since we are working with the flow polytope, $d(G)$ corresponds to the number of arcs $|A|$ in digraph $D$. The second part of the assumption holds if costs are nonnegative and if there is a maximum cost that can be encountered per task, which is natural in any real life setting.
    Furthermore, we use a generalized linear model with $\theta_a = \bfw^\top \bfphi_a$, where $\bfphi_a$ encodes some features of the connection $a$, related to the slack between the tasks that are the tail and the head of $a$. In applications, it is natural to assume that this slack is bounded, which gives Assumption~\ref{lip}.
\end{continuance}

\begin{continuance}{ex:scheduling}
    In the context of single machine scheduling, the set of feasible solutions depends only on the number $n$ of tasks. As a consequence, $G$ can be identified with $n$, and $\calG$ is the subset of~$\bbZ_+$ corresponding to instance sizes. Supposing that it is finite only amounts to having an upper bound on the number of tasks. Finally, $d(G) = n$.
\end{continuance}

\noindent
Under Assumption~\ref{A_emb}, we have the following decomposition of the law $\bbP_{X}$ of~$X$
\begin{equation}
\notag
    \bbE_{X} [g(X)]
        :=\int_{\cal X}g(\vx)\mathrm d\bbP_{X}(\vx)
        =\sum_{G\in\mathcal G}\mathds P[X\in \mathcal X_G]\int_{\calX_G}g(\vx)\mathrm d\bbP_{\xi_G}(\vx)\,,
\end{equation}
where $\bbP_{\xi_G}$ is the law of $\xi_G$, namely the conditional law of $\mathcal L(X|{X\in\mathcal X_G})$, and $g(\cdot)$ is any bounded continuous function. Denote by $\bbP_{X}^{(n)}$ its empirical version, given by
\begin{equation}
\notag
   \bbE_{X}^{(n)} [g(X)] := \frac1n\sum_{i\in[n]} g(X_i)
            =\frac1n\sum_{i\in[n]}\sum_{G\in\mathcal G}\mathds 1_{\{X_i\in \mathcal X_G\}}g(X_i).
\end{equation}

\paragraph{The Gaussian perturbation assumption:}
The framework of \citet{berthetLearningDifferentiablePerturbed2020} applies to perturbations~$Z(\cdot)$, which can be any isotropic law that is easy to simulate. However, to the best of our knowledge, and for every numerical paper in the related works, practitioners use Gaussian perturbations in their experiments. The Gaussian distribution has the additional advantage of making computations simpler in our proofs.
Throughout this \Cref{sec:setting_contributions}, we therefore consider that the perturbation random field $Z(\cdot)$ is a Gaussian random field satisfying the following assumption.

\medskip

\begin{assum}{Gauss}
\label{A3}
We assume that $\sqrt{d(G)}R(G)$ has a chi distribution with~$d(G)$ degrees of freedom, which is equivalent to
\begin{equation}
\notag
    \forall \vx\in\mathcal X_G\,,\quad
    \sqrt{d(G)}Z(\vx)\sim \mathcal N(0,\mathrm{Id}_{d(G)})\,.
\end{equation}
\end{assum}

\noindent
Under Assumptions~\ref{A_emb} and~\ref{A3}, we have the following decomposition of the law $\mathcal L(X,Z)$,
\begin{equation}
\label{eq=law_XZ}
    \bbE_{X,Z} \big[t(X,Z)\big]
        :=
        \sum_{G\in\mathcal G}\mathds P[X\in \mathcal X_G]\int_{\calX_G}
        \Big(\int_{\mathds R^{d(G)}}t(\vx,\vz)\varphi_{d(G)}(\vz)\mathrm d\vz\Big)\mathrm d\bbP_{\xi_G}(\vx)\, ,
\end{equation}
where $\bbP_{\xi_G}$ is the law of $\xi_G$, namely the conditional law of $\mathcal L(X|{X\in\mathcal X_G})$, $\varphi_{d(G)}$ is the Gaussian density of $\mathcal N(0,(1/d(G))\mathrm{Id}_{d(G)})$, and $t(\cdot)$ is any bounded continuous function.

\paragraph{Lipschitz feature map:} The following assumption is about the Lipschitz continuity of the feature map.

\begin{assum}{Lip}
    \label{lip}
    For all $\vx\in\calX$, the function $\vw\in\calW\mapsto\featw(\vx)\in\bbR^{d(\vx)}$ is $L_{\calW}$-Lipschitz continuous with a constant $L_{\calW}$ which does not depend on $\vx$.
\end{assum}

\begin{continuance}{ex:stoVSP}
    In the stochastic vehicle scheduling problem example, we use a generalized linear model with $\theta_a = \bfw^\top \bfphi_a$, where $\bfphi_a$ encodes some features of the connection $a$, related to the slack between the tasks that are the tail and the head of $a$. In applications, it is natural to assume that this slack is bounded, which gives our Lipschitz assumption.
\end{continuance}

\begin{remark}
When $\vw$ encodes the parameters of a neural network with smooth activation functions, one can compute the differential of the neural network with respect to $\vw$ at input point $\vx$, namely the gradient of $\vw\mapsto\featw(\vx)$. By the chain rule, one obtains that partial gradients~$\partial \featw(\vx)/\partial w_i$ are product of derivatives of the activation functions evaluated at layer values, of matrices of some of the parameters~$\vw$, and of the input vector~$\vx$. Since the layer values are themselves a composition of activation functions and product of matrices of some parameters~$\vw$ and input vector $\vx$, they are bounded by compactness of $\calW$ and $\calX$. Hence, partial gradients~$\partial \featw(\vx)/\partial w_i$ are also bounded by compactness. We thus deduce that \ref{lip} is satisfied.
\end{remark}

\paragraph{Existence of solutions:}
We assume that the optimization problem~\eqref{eq:optimal_risk} has a solution and that the minimum of the risk has a minimizer.

\begin{assum}{Sol}
    \label{A1}
    There exists $\vw^\star\in\calW$ solution to \eqref{eq:optimal_risk} and $\displaystyle \min_{\vw\in\calW}\calR_{n,\varepsilon_0}(h_{\vw})$ has a minimizer.
\end{assum}

\begin{prop}
    \label{pro:sol_OK}
    Let $\varepsilon_0$ be positive and let Assumptions \ref{A3} and \ref{lip} hold. Then, Assumption~\ref{A1} is~met.
\end{prop}

\begin{proof}
    When $\varepsilon_0>0$, the risk $\vw\mapsto\calR_{\varepsilon_0}(h_\vw)$ is continuous with respect to $\vw$ (see Lemma~\ref{lem:cont_risk}) and the set $\calW$ is compact. Therefore, the infimum in~\eqref{eq:optimal_risk} is attained.
\end{proof}

Having fixed the standard conditions, we now turn to the first term of the error decomposition in Theorem~\ref{thm:principal}: the bias incurred by regularizing the risk through perturbation. The next subsection introduces the central geometric quantity controlling this bias and derives an explicit bound.

\subsection{Regularization by perturbation and perturbation bias}
\label{sec:pertub_bias}

\subsubsection{The fan-crossing probability and the perturbation bias}
We now move on to the proof of Theorem~\ref{thm:principal}. A central role is played by the following quantity, which we call the \emph{fan-crossing probability}:
\begin{equation}
\label{eq:def_Vw}
    q_\vw(\lambda):=\bbE_{X}
                \mathbb P_Z\Big[\|Z(X)\|_2> \frac{\rho(\psi_{\vw}(X))}{\lambda}\,\big|\,X\Big]
            = \bbE_{X}
                \mathbb P_R\Big[{R(X)}>
                    \frac{\rho(\psi_{\vw}(X))}{\lambda}
                \,\big|\,X\Big]
                \,,
\end{equation}
for $\lambda>0$ and $\vw\in\calW$. By dominated convergence, note that $q_\vw(\lambda)$ tends to zero as $\lambda$ goes to zero and that~$q_\vw(\lambda)$ tends to one as $\lambda$ goes to infinity. Note also that the function $\lambda\mapsto q_\vw(\lambda)\in[0,1]$ is non-decreasing. Intuitively, $q_\vw(\lambda)$ measures the probability that a perturbation of scale $\lambda$ is large enough to push the direction $\psi_{\vw}(X)$ across a boundary of the normal fan, thereby changing the oracle solution. The perturbation bias is controlled by $q_\vw(\lambda)$, as the following proposition shows.

\medskip

\noindent
\begin{prop}
    \label{prop:bound_reg_w}
    Let $\varepsilon_0\geq 0$ and $\lambda>0$ be such that $\lambda\geq \varepsilon_0$. It holds that, for all $\vw\in\calW$, with $\operatorname{osc}$ defined in~\eqref{eq:def_osc},
    \[
    \big|\calR_\la(h_\vw)-\calR(h_\vw)\big|\leq 2 \operatorname{osc}(\fh) q_\vw(\lambda)\, \quad \text{and} \quad \big|\calR_\la(h_\vw)-\calR_{\varepsilon_0}(h_\vw)\big|\leq 4 \operatorname{osc}(\fh) q_\vw(\lambda)\,.
    \]
\end{prop}
\begin{proof}
    Recall that $p_0(\vy|\psi_{\vw}(\vx))\stackrel{\eqref{eq:def_p_0} }{=} \mathbb P[\psi_{\vw}(\vx)+\rho(\psi_{\vw}(\vx))U(\vx)\in\calF_\vy]$ where $U$ is uniform on $B(0,1)$ and

\[
    p_\lambda(\vy|\psi_{\vw}(\vx)) \stackrel{\eqref{eq:def_p_lambda} }{=} \mathbb P_Z[\psi_{\vw}(\vx)/\la+Z(\vx)\in\calF_\vy] = \mathbb P_Z[\psi_{\vw}(\vx)/\la+R(\vx)U(\vx)\in\calF_\vy]\,.
\]

\medskip

\noindent
    Note that
    \begin{align*}
        &\calR_\la(h_\vw)-\calR(h_\vw)\\
            &\stackrel{\eqref{eq:reg_risk_solution}}{=}
                \bbE_{X}\big[
                    \sum_{\vy\in\calY(X)} p_\la(\vy|\psi_{\vw}(X)) \fh(\vy, X)-\sum_{\vy\in\calY(X)} p_0(\vy|\psi_{\vw}(X)) \fh(\vy,X)
                \big]\\
            &=
            \bbE_{X,R,U}\big[
                \sum_{\vy\in\calY(X)} \underbrace{\Big(\ind\big(\psi_{\vw}(X)/\la+R(X)U(X)\in\calF_\vy\big) - \ind\big(\psi_{\vw}(X)+\rho(\psi_{\vw}(X))U(X)\in\calF_\vy\big)\Big)}_{\eqqcolon  A(X,\vy,R,U)} \fh(\vy,X)
            \big] \\
            &\stackrel{\eqref{eq:def_rho}}{=} \bbE_{X,R,U}\big[
                \sum_{\vy\in\calY(X)}
                \Big(\ind\big(R(X) \leq \rho(\psi_{\vw}(X))/\la\big) + \ind\big(R(X) > \rho(\psi_{\vw}(X))/\la\big)  \Big)
                A(X,\vy,R,U)
                \fh(\vy,X)
                \big]
    \end{align*}
    where we have introduced the random variable $A(X,\vy,R,U)$, using the definition of probabilities with indicators~$\ind$. By a zero-sum argument, one can replace $\fh(\vy, X)$ by $\fh(\vy, X)-\inf\fh$ in the calculation above.

    \medskip

    By definition of $\rho$, if $R(X) \leq \rho(\psi_{\vw}(X))/\la$ then $\ind\big(\psi_{\vw}(X)/\la+R(X)U(X) \in \calF_\vy\big)$ is equal to $\ind\big(\psi_{\vw}(X)+\rho(\psi_{\vw}(X))U(X)\in\calF_\vy\big)$. Hence $A(X,\vy,R,U) \ind\big(R(X) \leq \rho(\psi_{\vw}(X))/\la\big) = 0$.

    \medskip

    We therefore have:
    \begin{align*}
        \big|\calR_\la(h_\vw)-\calR(h_\vw)\big|
            &= \big|\bbE_{X,R,U}\big[
                \sum_{\vy\in\calY(X)}
                \ind\big(R(X) > \rho(\psi_{\vw}(X))/\la\big)
                A(X,\vy,R,U)
                (\fh(\vy,X)-\inf\fh)
                \big]\big| \\
            & \leq \bbE_{X,R,U}\big[ \sum_{\vy\in\calY(X)}
            \ind\big(R(X) > \rho(\psi_{\vw}(X))/\la\big) \, |A(X,\vy,R,U)| \, (\fh(\vy,X)-\inf\fh)
            \big] \\
            &\leq \bbE_{X,R,U}\big[
                \operatorname{osc}(\fh)  \ind\big(R(X) > \rho(\psi_{\vw}(X))/\la\big)
            \underbrace{\sum_{\vy\in\calY(X)}
            |A(X,\vy,R,U)|}_{\leq 2}
            \big] \\
            &\leq 2 \operatorname{osc}(\fh) \bbE_{X} \bbP_R\big(R(X) > \rho(\psi_{\vw}(X))/\la | X\big)=2 \operatorname{osc}(\fh) q_\vw(\lambda).
    \end{align*}
    Using that the function $\lambda\mapsto q_\vw(\lambda)\in[0,1]$ is non-decreasing, we have
    \begin{align*}
        \big|\calR_\la(h_\vw)-\calR_{\varepsilon_0}(h_\vw)\big|
            &\leq \big|\calR_\la(h_\vw)-\calR(h_\vw)\big|+\big|\calR(h_\vw)-\calR_{\varepsilon_0}(h_\vw)\big|\\
            & \leq 2 \operatorname{osc}(\fh) q_\vw(\lambda) + 2 \operatorname{osc}(\fh) q_\vw(\varepsilon_0) \\
            &\leq 4 \operatorname{osc}(\fh) q_\vw(\lambda)\,,
    \end{align*}
    which concludes the proof.
\end{proof}

\begin{coro}
Under Condition~\ref{A1}, it holds that
\[
\calR_\la(\calH)-\calR(\calH)\leq 2\operatorname{osc}(\fh) q_{\vw^\star}(\lambda)\,.
\]
\end{coro}
\begin{proof}
    Observe that $\mathcal R(\mathcal H)=\bbE_{X}\big[\sum_{\vy\in\calY(X)} p_0(\vy|\psi_{\vw^{\star}}(X)) \fh(\vy,X)\big]$ and remark that it holds
    \[
        \mathcal R_\la(\mathcal H)\leq\bbE_{X}\big[
        \sum_{\vy\in\calY(X)} p_\la(\vy|\psi_{\vw^{\star}}(X)) \fh(\vy, X) \big]\,,
    \]
    which gives the result by substracting the first equality.
\end{proof}

\medskip

\noindent
The remainder of this section is devoted to bounding the fan-crossing probability $q_\vw(\lambda)$. We introduce two complementary conditions: the Uniform Weak (UW) moment property, which yields a sub-linear bound, and the Uniformly Bounded Density (UBD) property, which yields a sharp linear bound.

\subsubsection{The uniform weak (UW) moment property}
The following property provides a general framework for controlling $q_\vw(\lambda)$.

\medskip

\begin{proper}{$\mathrm{UW}_{\varepsilon_0}$}
For all $\tau\in(0,1)$, there exists a positive constant $C_{\varepsilon_0,\tau}>0$ such that
\label{A_law}
\begin{equation}
\notag
    \forall\vw\in\calW\,,\
    \bbE_{X,Z}\bigg[
                    \Big(
                        \frac
                        {\rho(\psi_{\vw}(X)+\varepsilon_0Z(X))}{\sqrt{d(X)}}
                        \Big)^{-\tau}
                    \bigg]\leq C_{\varepsilon_0,\tau}\,.
\end{equation}
\end{proper}

\medskip

\begin{remark}
This property quantifies how likely the perturbed direction $\psi_{\vw}(X)+\varepsilon_0Z(X)$ is close to the boundary of some cone of the normal fan. By Markov's inequality, it holds
 \begin{align*}
 \forall t>0\,,\quad
    \bbP_{X,Z}\Big[{\rho(\psi_{\vw}(X)+\varepsilon_0Z(X))} \leq t\sqrt{d(X)}\Big]
 &=
    \bbP_{X,Z}\Big[
    \Big(
        \frac
        {\rho(\psi_{\vw}(X)+\varepsilon_0Z(X))}{\sqrt{d(X)}}
    \Big)^{-\tau} \geq t^{-\tau}\Big]\\
&\leq C_{\varepsilon_0,\tau} t^{\tau}\,,
 \end{align*}
whenever \ref{A_law} holds. It implies that
\begin{align*}
\forall t>0\,,\quad
    \bbP_{X,Z}\Big[{\rho(\psi_{\vw}(X)+\varepsilon_0Z(X))} \leq t\sqrt{d(X)}\Big]
    \leq \max(1,C_{\varepsilon_0,\tau} t) = \bbP\big[\mathcal U(0,1)\leq C_{\varepsilon_0,\tau} t \big]\,,
\end{align*}
where $\mathcal U(0,1)$ is the uniform distribution on $(0,1)$ independent of $(X,Z)$. Hence, one gets the stochastic ordering
\[
{\rho(\psi_{\vw}(X)+\varepsilon_0Z(X))}\leq \frac{\sqrt{d(X)}}{C_{\varepsilon_0,\tau}}\, \mathcal U(0,1)\,.
\]
If we further assume that $\lim\sup_{\tau\in(0,1)} C_{\varepsilon_0,\tau}<\infty$, this latter condition would have been equivalent to Property~\ref{A_law} by the above reasoning. Hence, Property~\ref{A_law} guarantees that the distance to the boundary of normal cones of the perturbed direction $\psi_{\vw}(X)+\varepsilon_0Z(X)$ is stochastically dominated by some uniform distribution and cannot concentrate too much mass close to $0$.
\end{remark}

Nevertheless, Property~\ref{A_law} requires $\tau \in (0,1)$. To reach $\tau=1$, an alternative is to directly require a continuous uniformly bounded density, as discussed in \Cref{sec:UBD}; this also avoids the issues associated with concentrating mass at the boundary of the normal fans.

\subsubsection{The uniformly bounded density (UBD) property}\label{sec:UBD}
Under stronger regularity, the fan-crossing probability admits a sharper, linear bound. We single out the following condition.

\begin{proper}{UBD}
\label{A_UBD}
For each partition $G\in\calG$, there exist a compact set $K_G\subset\bbR^{d(G)}$ and a constant $C_{\psi,G}>0$ such that, for all $\vw\in\calW$, the conditional push-forward law of $\psi_\vw(X)$ given $X\in\calX_G$ is absolutely continuous with respect to the Lebesgue measure on $\bbR^{d(G)}$, with density uniformly bounded by $C_{\psi,G}$, and supported on $K_G$.
\end{proper}

\begin{remark}
Property~\ref{A_UBD} is a strengthening of Case~$\rm{(i)}$ of Proposition~\ref{prop:prop3} below: it requires the push-forward law of the learned directions to have a uniformly bounded density on a compact support. In particular, \ref{A_UBD} implies \ref{A_law} for $\varepsilon_0=0$ (see Proposition~\ref{prop:prop3}).
\end{remark}

\begin{remark}[Examples satisfying Property~\ref{A_UBD} and other cases of Proposition~\ref{prop:prop3}]
\label{rem:examples_bounded_conditional}
Practically, features are generally bounded, and assuming that $X$ has a compact support is not restrictive.
For instance, in many transportation and routing problems, the instance $X$ consists of continuous exogenous features (e.g., meteorological data, continuous traffic density) naturally bounded within a physical range.

However, assuming a density of $X$ that is uniformly bounded and absolutely continuous with respect to the Lebesgue measure might be more restrictive given that some labels might be integer-valued.
 If the model $\psi_\vw(X)$ predicting the edge costs $\bftheta \in \bbR^{|E|}$ is a sufficiently rich (full-rank) linear or smooth mapping, the induced distribution of the predicted costs naturally avoids singular concentrations and inherits a bounded density.
Property~\ref{A_UBD}---having a uniformly bounded density on a compact set---is therefore realistic for several standard architectures:
\begin{itemize}
    \item \textbf{Generalized Linear Models:} Consider a generalized linear model where $\bftheta = \Phi_\vx \vw$ for a given matrix of features, such as the architecture used in the stochastic vehicle scheduling problem (Example~\ref{ex:stoVSP}). Property~\ref{A_UBD} is satisfied when every feature is continuous and upper bounded, and the number of features is rich enough (i.e., full rank) so that the mapping does not collapse into a lower-dimensional subspace. In the stochastic vehicle scheduling problem, continuous features reflect properties of the delay or the distance, while discrete features can indicate combinatorial properties such as the number of arcs departing from the tail of an arc $a$.
    When the number of features is small, the image of the instances through the model may be included in a hyperplane, which leads to Case~${\rm (iii)}$ of Proposition~\ref{prop:prop3}. Case~${\rm (ii)}$ corresponds to features that take integer values and are finite.
    \item \textbf{Deterministic Neural Networks (via the Coarea Formula):} For a deterministic neural network with smooth activations, the coarea formula guarantees a bounded push-forward density provided the input~$X$ has a bounded density and the network's Jacobian is uniformly lower-bounded ($|J\psi_\vw|(\vx) \ge \alpha > 0$), see Proposition~\ref{prop:coarea_density}. However, this precludes representations that collapse onto lower-dimensional manifolds, a usual phenomenon in deep learning. Using a sub-manifold instead of a hyperplane, Case~${\rm (iii)}$ of Proposition~\ref{prop:prop3} enables us to deal with such neural networks.
    \item \textbf{Stochastic Neural Networks and Randomized Smoothing:} To easily bypass the topological limitations of deterministic networks, one can employ randomized smoothing. By injecting a small, independent Gaussian noise $\xi \sim \mathcal{N}(0, \sigma^2 I_{d(X)})$ at the final layer, the resulting model $\tilde{\psi}_\vw(X) = \psi_\vw(X) + \xi$ produces a push-forward law that is the convolution of the deterministic network's output distribution with a Gaussian. As we formally prove in Lemma~\ref{lem:bounded_density_convolution} (Appendix~\ref{app:technical_lemmas}), this operation acts as a universal mollifier: it unconditionally guarantees that the density of $\tilde{\psi}_\vw(X)$ is absolutely continuous and uniformly bounded by $(2\pi\sigma^2)^{-d(X)/2}$ for all $\vw \in \calW$, regardless of the network's internal Jacobian or potential manifold collapse.
\end{itemize}
\end{remark}
\subsubsection{Discussion on the UW and UBD properties}
\label{sec:UW_examples}

\paragraph{UW is always met for positive regularization:}
We have the following proposition showing that, when $\varepsilon_0>0$, Property~\ref{A_law} holds with a constant $C_{\varepsilon_0,\tau}$ that depends only on $\varepsilon_0$ and $\tau$.
\begin{prop}
\label{prop:UW_conv}
If $\varepsilon_0>0$ then, under Assumptions~\ref{A_emb} and~\ref{A3}, Property \ref{A_law} holds and
    \[
        C_{\varepsilon_0,\tau}\leq \Big(\int_{\mathds R}|t|^{-\tau}\varphi(t)\mathrm dt\Big)\times
        \varepsilon_0^{-\tau}\times
        \sum_{G\in\mathcal G}
        \mathds P[X\in \mathcal X_G]\,|\mathcal Y(G)|^2\,d(G)^{\tau}\,
        \,,
    \]
where $\varphi$ is the standard Gaussian density on $\R$.
\end{prop}

\begin{proof}
In view of \eqref{eq=law_XZ} and without loss of generality, we consider that $X\in\mathcal X_G$ almost surely, for~$G$ fixed. Let~$H$ be a hyperplane and $\operatorname{dist}(\theta,H)$ be the Euclidean distance between $\theta$ and $H$. By~\eqref{eq=law_XZ}, note that
\begin{align*}
\bbE_{X,Z}\bigg[
                    \Big(
                        \frac
                        {\operatorname{dist}(\psi_{\vw}(X)+\varepsilon_0Z(X),H)}{\sqrt{d(X)}}
                        \Big)^{-\tau}
                    \bigg]
&=
\bbE_{\xi_G,Z}\bigg[
                    \Big(
                        \frac{d(G)^{\frac\tau2}}
                        {\operatorname{dist}(\psi_{\vw}(\xi_G)+\varepsilon_0Z(\xi_G),H)^{\tau}}
                        \Big)
                    \bigg]\\
&=
d(G)^{\frac\tau2}\,\bbE_{\xi_G}\bigg[
                    \bbE_{Z}\Big(
                        \frac{1}
                        {\operatorname{dist}(\psi_{\vw}(\xi_G)+\varepsilon_0Z(\xi_G),H)^{\tau}}
                        \Big|\xi_G\Big)
                    \bigg]\,.
\end{align*}
Conditionally to $\xi_G$, observe that ${\operatorname{dist}(\psi_{\vw}(\xi_G)+\varepsilon_0Z(\xi_G),H)^{\tau}}$ is the distance between the hyperplane~$H$ and a Gaussian with mean $\psi_{\vw}(\xi_G)$ and variance $(\varepsilon_0^2/d(G))\mathrm{Id}_{d(G)}$. This distance is
\[
{\operatorname{dist}(\psi_{\vw}(\xi_G)+\varepsilon_0Z(\xi_G),H)^{\tau}}=|\langle n_H\,,\, \psi_{\vw}(\xi_G)+\varepsilon_0Z(\xi_G)\rangle|^{\tau}
\]
where $n_H$ is a unitary orthogonal vector to $H$. Conditionally to $\xi_G$, $\langle n_H\,,\, \psi_{\vw}(\xi_G)+\varepsilon_0Z(\xi_G)\rangle$ is a Gaussian random variable with mean $\langle n_H\,,\, \psi_{\vw}(\xi_G)\rangle$ and variance $\varepsilon_0^2/d(G)$.

Observe that the level sets of $t\mapsto |t|^{-\tau}$ are symmetric convex bodies, namely intervals of the form $[-a,a]$ for $a>0$. By Anderson's lemma \citep[Theorem 2.4.4]{gine2021mathematical}, it holds that
\begin{equation}
\label{eq:case_t_minus_tau}
\bbE_{Z}\Big(
                        \frac{1}
                        {\operatorname{dist}(\psi_{\vw}(\xi_G)+\varepsilon_0Z(\xi_G),H)^{\tau}}
                        \Big|\xi_G\Big)
                        \leq
\bbE_{Z}\Big(
                        \frac{1}
                        {\operatorname{dist}(\varepsilon_0Z(\xi_G),H)^{\tau}}
                        \Big|\xi_G\Big)
=
\frac{d(G)^{\frac\tau2}}{\varepsilon_0^\tau}
\int_{\mathds R}|t|^{-\tau}\varphi(t)\mathrm dt\,.
\end{equation}
We deduce that
\[
\bbE_{X,Z}\bigg[
                    \Big(
                        \frac
                        {\operatorname{dist}(\psi_{\vw}(X)+\varepsilon_0Z(X),H)}{\sqrt{d(X)}}
                        \Big)^{-\tau}
                    \bigg]
                        \leq
\frac{d(G)^{\tau}}{\varepsilon_0^\tau}\int_{\mathds R}|t|^{-\tau}\varphi(t)\mathrm dt<\infty
\]

\medskip

\noindent

Then observe that, for $Q$ a polyhedral cone with non-empty interior, we have that $Q$ is the intersection of halfspaces delimited by hyperplanes $H_1,\ldots,H_D$. Note also that
\begin{equation}\label{eq:comparison_dist}
    \frac{1}{\operatorname{dist}(\theta,Q)^\tau} = \frac{1}{(\min\left\{\operatorname{dist}(\theta,H_i)\right\})^\tau}  = \max\left\{\frac{1}{\operatorname{dist}(\theta,H_i)^\tau}\right\} \leq \frac{1}{\operatorname{dist}(\theta,H_1)^\tau} + \ldots + \frac{1}{\operatorname{dist}(\theta,H_D)^\tau}
\end{equation}
and thus $\theta\mapsto\frac{1}{\operatorname{dist}(\theta,Q)^\tau}$ is integrable with respect to the law of $\psi_{\vw}(X)+\varepsilon_0Z(X)$ conditional to $X\in \mathcal X_G$. Finally, observe that $\theta\mapsto \rho(\theta)$ is the distance to the boundaries of polyhedral cones. We have shown the bound for each of the ${\operatorname{dist}(\theta,H_k)^\tau}$, and there are at most $|\mathcal Y(G)|^2$ such hyperplanes counting the number of different couples~$(\vy,\vy')$ in the definition of the normal cone. This concludes the proof.
\end{proof}

\paragraph{The no perturbation case:}
We now discuss Property~\ref{A_law} when $\varepsilon_0=0$. The argument of the proof of Proposition~\ref{prop:UW_conv} gives a scheme of proof for the case $\varepsilon_0=0$ which we detail here. The proof of Proposition~\ref{prop:UW_conv} is based on two arguments: (i) reduction to the case of the study of one hyperplane thanks to~\eqref{eq:comparison_dist}, and (ii) integrability of $t\mapsto|t|^{-\tau}$ with respect to a family of laws described by $\psi_{\vw}(X)$. This latter point is the one harder to handle, as it depends on the law of $X$ and the family of mappings $\psi_{\vw}$. This was achieved in the proof of Proposition~\ref{prop:UW_conv} by Anderson's lemma \eqref{eq:case_t_minus_tau} to derive an upper bound that depends only on the regularization law (here the Gaussian distribution). In the no perturbation case, we cannot invoke this latter argument since the upper bound diverges towards infinity as $\varepsilon_0$ goes to zero. Nevertheless the same scheme of proof leads to the following result.

\begin{prop}
\label{prop:prop3}
Consider the family of laws
\[
\mathds L:=
\big\{
\mathcal{L}(\psi_{\vw}(X))\ :\ \vw\in\calW
\big\}
\]
where $\mathcal{L}(\psi_{\vw}(X))$ denotes the law of $\psi_{\vw}(X)$. Under Assumption~\ref{A_emb} and without loss of generality, we consider that~$d(X)$ is constant almost surely, say $d(X)=d$. {\bf If} the family of laws $\mathds L$ satisfies one of the four scenarii:
\begin{itemize}
    \item[\rm{(i)}] Property~\ref{A_UBD} holds, i.e., there exists a compact $K\subset\mathds R^d$ such that the supports of the laws of $\mathds L$ are included in~$K$ and these laws are absolutely continuous, with respect to the Lebesgue measure, with uniformly bounded densities;
    \item[\rm{(ii)}] There exists a finite set $K\subset\mathds R^d$ which does not intersect the boundaries of the normal cones such that the laws of $\mathds L$ are finite with support included in~$K$;
    \item[\rm{(iii)}] There exists a compact Riemannian submanifold  $K\subset\mathds R^d$ (with metric induced by the Euclidean metric) such that the supports of the laws of $\mathds L$ are included in $K$ and these laws are absolutely continuous, with respect to the uniform Riemannian measure, with uniformly bounded densities;
    \item[\rm{(iv)}] The laws of $\mathds L$ are a mixture of an absolutely continuous law satisfying {\rm{(i)}}, a discrete law satisfying {\rm{(ii)}}, and a singular continuous law satisfying \rm{(iii)},
\end{itemize}
{\bf then} Property~\ref{A_law} holds for $\varepsilon_0=0$.
\end{prop}
\begin{proof}
We follow the same lines as in the proof of Proposition~\ref{prop:UW_conv}. Note that
\begin{itemize}
    \item $\theta\mapsto \rho(\theta)$ is the distance to the boundaries of polyhedral cones;
    \item $\frac{1}{\operatorname{dist}(\theta,Q)^\tau} \leq \frac{1}{\operatorname{dist}(\theta,H_1)^\tau} + \ldots + \frac{1}{\operatorname{dist}(\theta,H_D)^\tau}$ as in \eqref{eq:comparison_dist}, where $Q$ is a polyhedral cone which is the intersection of half-spaces delimited by $H_1,\ldots,H_{D_G}$ with $D_G\leq|\mathcal Y(G)|^2$;
    \item $\theta\mapsto\frac{1}{\operatorname{dist}(\theta,H_k)^\tau}$ is integrable with respect to the uniform measure on compacts for $k\in[D_G]$;
    \item we have an upper bound on the \ref{A_law} constant for $\varepsilon_0=0$,
    \[
        C_{0,\tau}\leq
        \sum_{G\in\mathcal G}
        \mathds P[X\in \mathcal X_G]\,
        \sum_{k=1}^{D_G}
            d(G)^{\frac\tau2}\,
            \bbE_{\xi_G}\Big(
                        \frac{1}
                        {\operatorname{dist}(\psi_{\vw}(\xi_G),H_k)^{\tau}}
                        \Big)
        \,,
    \]
    by~\eqref{eq=law_XZ}.
\end{itemize}
Hence it remains to prove that $\bbE_{\xi_G}\Big(\frac{1}{\operatorname{dist}(\psi_{\vw}(\xi_G),H_k)^{\tau}}\Big)$ is finite. Case $\rm{(ii)}$ is clear and Case $\rm{(iv)}$ also as soon as the property can be proven for Cases $\rm{(i)}$ and $\rm{(iii)}$. In these latter cases, there exists a compact $K\subset\mathds R^d$ such that the supports of the laws of $\mathds L$ are included in $K$ and these laws are absolutely continuous, with respect to the uniform (Riemannian) measure, with uniformly bounded densities, say by a constant $C_{\psi}$. Note that
\begin{align*}
\bbE_{\xi_G}\Big(\frac{1}{\operatorname{dist}(\psi_{\vw}(\xi_G),H_k)^{\tau}}\Big)
&= \int_{K}  \frac{1}{\operatorname{dist}(\vtheta,H_k)^{\tau}} f_{\psi_{\vw}}(\vtheta)\mathrm{d}\vtheta\\
&\leq C_\psi \int_{K}  \frac{1}{\operatorname{dist}(\vtheta,H_k)^{\tau}} \mathrm{d}\vtheta
\end{align*}
and the last integral is finite since $K$ is compact.
\end{proof}

\medskip

It is out of the scope of this paper to derive a comprehensive study of proving the property \ref{A_law} for $\varepsilon_0=0$ in a general manner. It is sensible to assume that $\mathds L$ has bounded support and the integrability of $t\mapsto|t|^{-\tau}$ can be shown by proving that the densities of $\mathds L$ are uniformly bounded (Property~\ref{A_UBD}). The latter argument is more intricate to derive and would be better studied for specific applications as it depends on the specific modelling of $X$ and the choice of $\psi_{\vw}$. Instead we provide the following observations based on a change of variable formula and the coarea formula:
\begin{itemize}
    \item First observe that $\vtheta=\psi_\vw(\vx)$ is the vector appearing in the objective of the linear program \eqref{eq:ylinearProblem}. Note that for any minimizer $\hat\vy_\vx(\vtheta)$, it holds that $\hat\vy_\vx(t\vtheta)=\hat\vy_\vx(\vtheta)$ for $t>0$. Hence, without loss of generality, one can design $\psi_{\vw}(X)$ so that its Euclidean norm is less than one, say. It follows that the push-forward laws~$\mathds L$ can be assumed to have supports contained in the Euclidean ball, or more generally in some compact set.
    \item Second, the regularity of the push-forward laws $\mathcal{L}(\psi_{\vw}(X))$ depends on the regularities of $X$ and of~$\psi_\vw$. As shown in Proposition~\ref{prop:coarea_density}, if $\psi_\vw : \bbR^p \to \bbR^d$ (with $p \ge d$) is Lipschitz non-degenerate with Hausdorff $\calH^{p-d}$ bounded level sets then, for all $\vw \in \calW$, the push-forward law of $\psi_\vw(X)$ satisfies Property~\ref{A_UBD}: it is absolutely continuous with respect to the $d$-dimensional Lebesgue measure, and its probability density function $g_\vw(\theta)$ is uniformly bounded.
\end{itemize}

\subsubsection{Perturbation bias controls}
Combining Proposition~\ref{prop:bound_reg_w} and Proposition~\ref{prop:Gauss_V} in the Appendix, we get the following theorem.

\begin{theo}
\label{thm:reg_excess}
Under~\ref{A_law} and \ref{A3}, the following holds true. Let $\varepsilon_0\geq 0$ and $\lambda>0$ be such that $\lambda\geq \varepsilon_0$. Let $\tau\in(0,1)$. There exists a constant $C>0$ $($that may depend on the constant~$C_{\varepsilon_0,\tau}$, appearing in \ref{A_law}, and $\operatorname{osc}(\fh))$ such that, for all $\vw\in\calW$,
\[
\big|\calR_\la(h_\vw)-\calR_{\varepsilon_0}(h_\vw)\big|\leq C \,\lambda^{\tau}\mathrm{polylog}(\la)\,,
\]
where $\mathrm{polylog}(\la)$ is a polynomial logarithm term.
\end{theo}

\begin{proof}
    By Proposition~\ref{prop:bound_reg_w}, one has $\big|\calR_\la(h_\vw)-\calR_{\varepsilon_0}(h_\vw)\big|\leq 4 \operatorname{osc}(\fh) q_\vw(\lambda)$. Now, by Proposition~\ref{prop:Gauss_V},
    it holds that $q_\vw(\la)=\mathcal C\lambda^{\tau}\mathrm{polylog}(\la)$, under Assumption \eqref{A2bis}. Note that this latter is implied by Assumption~\ref{A_law}, hence the result.
\end{proof}

As a consequence of Theorem~\ref{thm:reg_excess}, the perturbation bias scales as $\mathcal{O}(\lambda^\tau \mathrm{polylog}(\la))$ because the Uniform Weak property inherently restricts the exponent to $\tau \in (0,1)$ to guarantee the integrability of the inverse distance moments around the boundaries. However, if the learned directions satisfy Property~\ref{A_UBD}, we can directly bound the probability mass near the boundaries. This geometric perspective yields a sharper, linear bound on the perturbation bias.

\begin{prop}[Linear Perturbation Bias]
\label{prop:linear_bias_bound}
Suppose Assumption~\ref{A_emb} holds and that the perturbation $Z$ satisfies $\bbE_Z[\|Z(\vx)\|_2] < \infty$ for all $\vx$ $($which holds under Assumption~\ref{A3}$)$.
Furthermore, assume that Property~\ref{A_UBD} holds.

Then, there exists a constant $C > 0$ independent of $\la$ such that for all $\vw \in \calW$,
\[
\big|\calR_\la(h_\vw)-\calR(h_\vw)\big| \le C \la \,.
\]
\end{prop}
This sharp bound has profound implications for the \emph{statistical-to-computational trade-off}. By guaranteeing a faster theoretical decay of the perturbation bias, practitioners can choose a comparatively larger smoothing parameter $\la$ for a given target excess risk. Because the $k$-SoS optimization error and the empirical process error both scale inversely with $\la$, being able to use a larger $\la$ directly lowers the required sample size $n$ and reduces the number of sampled points $M$ needed to mitigate the curse of dimensionality over $\calW$.
\begin{proof}
By Proposition~\ref{prop:bound_reg_w}, we have $\big|\calR_\la(h_\vw)-\calR(h_\vw)\big| \le 2 \operatorname{osc}(\fh) q_\vw(\lambda)$ and we recall that it holds $q_\vw(\lambda) = \bbP_{X, Z}\big( \rho(\psi_\vw(X)) < \lambda \|Z(X)\|_2 \big)$.
Using the partition $\calG$ from Assumption~\ref{A_emb}, we decompose the probability over the instance space:
$$ q_\vw(\lambda) = \sum_{G \in \calG} \bbP(X \in \calX_G) \bbE_{Z} \Big[ \bbP_{X|G} \big( \rho(\psi_\vw(X)) < \lambda \|Z(X)\|_2 \big| Z \big) \Big] \,, $$
where we used the independence of $X$ and $Z$ to condition on $Z$.

Fix a partition $G \in \calG$. For any $\vx \in \calX_G$, the normal fan boundaries consist of a fixed, finite union of hyperplanes $H_{1,G}, \dots, H_{D_G, G}$, where $D_G \le |\calY(G)|^2$. Applying the union bound, we obtain for any $r > 0$:
$$ \bbP_{X|G} \big( \rho(\psi_\vw(X)) < r \big) \le \sum_{k=1}^{D_G} \bbP_{X|G} \big( \mathrm{dist}(\psi_\vw(X), H_{k,G}) < r \big) \,. $$

For a specific hyperplane $H_{k,G}$, the probability of falling within an $r$-neighborhood is the integral of the conditional density $f_{\psi_\vw, G}(\theta)$ over this region intersected with the compact support $K_G$. Because the density is bounded by $C_{\psi, G}$, we have:
$$ \bbP_{X|G} \big( \mathrm{dist}(\psi_\vw(X), H_{k,G}) < r \big) \le C_{\psi, G} \int_{K_G \cap \{ \theta \,:\, \mathrm{dist}(\theta, H_{k,G}) < r \}} \mathrm{d}\theta \,. $$

Let $V_{\max, G}$ be the maximal $(d(G)-1)$-dimensional cross-sectional volume of $K_G$. The volume of the $r$-neighborhood of $H_{k,G}$ restricted to $K_G$ is bounded by the volume of a cylinder with base area $V_{\max, G}$ and height~$2r$. Thus:
\[
\bbP_{X|G} \big( \mathrm{dist}(\psi_\vw(X), H_{k,G}) < r \big) \le 2 r C_{\psi, G} V_{\max, G} \,.
\]
\noindent
Summing over all $D_G$ hyperplanes gives a linear bound for the fixed partition:
$$ \bbP_{X|G} \big( \rho(\psi_\vw(X)) < r \big) \le (2 D_G C_{\psi, G} V_{\max, G}) r \eqqcolon L_G r \,. $$

\noindent
Substitute $r = \lambda \|Z(X)\|_2$ back into the expectation over $Z$:
$$ \bbE_{Z} \Big[ \bbP_{X|G} \big( \rho(\psi_\vw(X)) < \lambda \|Z(X)\|_2 \big| Z \big) \Big] \le L_G \la \bbE_Z [\|Z(X)\|_2 \mid X \in \calX_G] \,. $$
Under Assumption~\ref{A3}, let $M_{Z,G} \coloneqq \bbE_Z[\|Z(X)\|_2 \mid X \in \calX_G] < \infty$.

\medskip

Summing over all partitions $G \in \calG$, we bound the overall perturbation bias:
$$ q_\vw(\lambda) \le \la \sum_{G \in \calG} \bbP(X \in \calX_G) L_G M_{Z,G} \eqqcolon L \la \,, $$
where $L$ is a finite constant since $\calG$ is a finite partition.
Combining this with Proposition~\ref{prop:bound_reg_w} gives:
$$ \big|\calR_\la(h_\vw)-\calR(h_\vw)\big| \le 2 \operatorname{osc}(\fh) L \la \,, $$
which concludes the proof by setting $C = 2 \operatorname{osc}(\fh) L$.
\end{proof}

\begin{remark}[Universality of the bounded density via Randomized Smoothing]
\label{rem:randomized_smoothing_fix}
For arbitrary deterministic statistical models $\psi_\vw$, guaranteeing a uniformly bounded density over a compact parameter space $\calW$ as done in \ref{A_UBD} is topologically restrictive. If the instance distribution $X$ is discrete or supported on a lower-dimensional space, or if the model parametrization permits rank-deficient mappings (e.g., $\vw = \mathbf{0}$, or decision trees outputting finite discrete sets), the learned representations $\psi_\vw(X)$ will concentrate on lower-dimensional sub-manifolds in $\bbR^{d(G)}$. In such cases, the push-forward law is singular with respect to the Lebesgue measure, and the bounded density assumption fails.\\
\indent
However, this theoretical bottleneck is \emph{universally resolved} by adopting a randomized smoothing approach. As recalled in Lemma~\ref{lem:bounded_density_convolution} (Appendix~\ref{app:technical_lemmas}), explicitly injecting an independent Gaussian perturbation $\xi \sim \mathcal{N}(0, \varepsilon_0^2 \mathrm{Id}_{d(G)})$ to the model's output yields a smoothed direction $\tilde{\psi}_\vw(X) = \psi_\vw(X) + \xi$. The resulting convolution unconditionally guarantees that the push-forward law admits a probability density uniformly bounded by $C_{\psi, G} = (2\pi\varepsilon_0^2)^{-d(G)/2}$ for all $\vw \in \calW$. Consequently, the sharp linear perturbation bias $\mathcal{O}(\la)$ is always accessible for any arbitrary combinatorial prediction architecture simply by incorporating a baseline stochasticity. In this case, note that $\mathcal{E}_{0}(\calH)$ of \eqref{eq:linear_bias_excess} given for $\tilde{\psi}_\vw$ is equal to $\mathcal{E}_{\varepsilon_0}(\calH)$ given for ${\psi}_\vw$.\\
\indent
That said, this universality comes at a cost hidden in the constants. The constant~$C$ in Equation~\eqref{eq:linear_bias_excess} depends linearly on the density bound~$C_{\psi, G} = (2\pi\varepsilon_0^2)^{-d(G)/2}$, which explodes exponentially with the ambient dimension $d(G)$ as $\varepsilon_0 \to 0$. Similarly, the fractional moment constant from Proposition~\ref{prop:UW_conv} diverges as $C_{\varepsilon_0, \tau} \propto \varepsilon_0^{-\tau}$. Consequently, the guarantee of~\eqref{eq:linear_bias_excess} degrades for an artificially smoothed deterministic model unless $\varepsilon_0$ is kept fixed, which in turn freezes the misspecification error $\mathcal{E}_{\varepsilon_0}(\calH)$.
\end{remark}
Ultimately, the sharp rates are most practically meaningful for models where the conditional density naturally avoids singular concentrations---such as the contextual operations research settings highlighted in Remark~\ref{rem:examples_bounded_conditional}---without needing an artificial noise injection.

\subsection{Empirical process control}
\label{sec:ERM}
We now turn our attention back to statistical learning guarantees. We consider the empirical regularized risk \eqref{eq:emp_reg_risk_solution} and
\begin{equation}
    \notag
    \calR_{n,\la}(\calH):=\min_{\vw\in \calW}
    \Big\{\frac1n  \sum_{i=1}^n\bbE_Z
        \big[
        \fh\big(\hat \vy_{X_i}(\featw(X_i)+\la Z(X_i)), X_i\big)
        \big]
    \Big\}\,.
\end{equation}
We would like to control the random variable
\[
    \Delta_n=\sup_{\vw\in \calW}\big|\calR_{\la}(h_{\vw})-\calR_{n,\la}(h_{\vw})\big|\,.
\]
We are going to use a Bernstein inequality for bounded sub-Gaussian random variables. In this analysis, we use the oscillation of $\fh$, recalled in Assumption~\ref{A_emb}, and Dudley's entropy integral, defined as
\begin{align}
\notag
   \calI_{\calW}&:=\int_0^\infty\!\! \sqrt{\log \calN(\calW,\|\cdot\|,\varepsilon)}\,\, \mathrm d\varepsilon\,,
\end{align}
where $\calN(\calW,\|\cdot\|,\varepsilon)$ is the $\varepsilon$-covering number of $\calW$ with respect to $\|\cdot\|$; for further details, the reader may refer to \cite[Chapter 5]{wainwright2019high}.

\begin{remark}[Example]
    Let $\mathds B$ be the unit ball of $\|\cdot\|$ and assume that there exists $R>0$ such that $\calW\subset R\mathds B$. The volume ratios lemma \cite[Lemma 5.7]{wainwright2019high} states that
    \[
    \calI_{\calW}\leq C {d_\calW} \log (R+1/R)\,,
    \]
    where $C>0$ is a universal constant.
\end{remark}

\begin{theo}
\label{thm:A3_statistical}
Let Assumptions~\ref{lip} and~\ref{A3} hold. Assume that $\displaystyle d(\calX):=\max_{\vx\in\calX} d(\vx)<\infty$ $($which is implied by~\ref{A_emb}$)$. Then, for all $\delta\in(0,1)$, it holds that
\begin{equation}
    \label{eq:bound_A3}
    \Delta_n=\sup_{\vw\in \calW}
    \big|\calR_{\la}(h_{\vw})-\calR_{n,\la}(h_{\vw})\big|
    \le
    \frac{\operatorname{osc}(\fh)}{\la \sqrt n}
    \Bigg(
    {(\ln{2})^{\nicefrac{-3}{4}}}\, L_{\calW}\, \calI_{\calW}\,\sqrt{d(\calX)}
    +
    4\sqrt{\ln{\dfrac{8}{\delta}}}
    \Bigg)\,,
\end{equation}
with probability higher than $1-\delta$.
\end{theo}

\begin{proof}
Bounding $\Delta_n$ is related to studying the zero-mean random variables
\begin{equation}
\label{eq:rv_empirical_error}
    V_{i,\vw}:= \frac{1}{n}\Big[ \sum_{\vy\in \calY(X_i)}  p_\la(\vy|\featw(X_i))\fh(\vy,X_i) - \bbE_{(X,Y)}\big[p_\la(Y|\featw(X))\fh(Y,X)\big]\Big]\,,
\end{equation}
which depend on the random samples $X_i\sim \bbP_X$. Indeed $\Delta_n=\sup_{\vw\in\calW} \big|\sum_{i\in[n]} V_{i,\vw} \big|$ for a given draw. Rather than studying the absolute value, we focus on $Z:=\sup_{\vw\in\calW} \sum_{i\in[n]} V_{i,\vw}$ and use a symmetrization argument. Since $\Delta_n=\max(\sup_{\vw\in\calW} \sum_{i\in[n]} V_{i,\vw},-\inf_{\vw\in\calW} \sum_{i\in[n]} V_{i,\vw})$, focusing on the first term, the second one being dealt with similarly, we will show that $Z$ is close to its mean $\bbE Z$ with high probability (through a Bernstein inequality) and then bound the mean $\bbE Z$ by $\calO(\nicefrac{1}{\sqrt{n}})$ (through Hoeffding and Dudley's inequalities).

Let~$(X'_i)\sim \bbP_X^{\otimes n}$ be drawn independently from~$(X_i)_{i\in[n]}$ and define~$V'_{i,\vw}$ as in Equation \eqref{eq:rv_empirical_error}. Then~$(V_{i,\vw})_{i\in[n]}$ and~$(V'_{i,\vw})_{i\in[n]}$ are independent with the same distribution.
\begin{lem}
\label{lem:variance_bound}
    We have that almost surely
    \begin{equation}\label{eq:variance_bound}
        \sup_{\vw\in\calW} \sum_{i\in[n]} (V_{i,\vw}-V'_{i,\vw})^2 \le \frac{(\operatorname{osc}(\fh))^2}{n}\,.
    \end{equation}
\end{lem}
\begin{proof} We use the fact that $p_\la \in [0,1]$ to obtain that
\begin{align*}
    V_{i,\vw}-V'_{i,\vw}
        &=\frac{1}{n}
            \Big[ \sum_{\vy\in \calY(X_i)}  p_\la(\vy|\featw(X_i))\fh(\vy,X_i) - \sum_{\vy\in \calY(X'_i)}  p_\la(\vy|\featw(X'_i))\fh(\vy,X'_i)\Big]
            \\
        &\le \frac{1}{n}
            \Big[ \sup_{\vx,\vx'\in\calX, \, \vy,\vy' \in \calY} \fh(\vy,\vx)-\fh(\vy',\vx')\Big] = \frac{\operatorname{osc}(\fh)}{n}\,.
\end{align*}
Since $(V'_{i,\vw}-V_{i,\vw})$ is equal in law to $(V_{i,\vw}-V'_{i,\vw})$, this yields \eqref{eq:variance_bound}.
\end{proof}
Owing to Lemma \ref{lem:variance_bound}, we can thus apply \citet[Theorem 12.3, Page 333]{Boucheron_Lugosi_Massart_2013} since
\[
    \bbE_{X'_i\sim\bbP_X}[\sup_{\vw\in\calW} \sum_{i\in[n]} (V_{i,\vw}-V'_{i,\vw})^2]\le \frac{(\operatorname{osc}(\fh))^2}{n}
    \,.
\]
Consequently, for all $t\ge 0$, we have
\begin{equation}\label{eq:A3_BLM}
    \bbP[Z\ge \bbE Z + t]\le 4 e^{-\frac{n t^2}{16 (\operatorname{osc}(\fh))^2}}.
\end{equation}

\medskip

We are now going to show that $\bbE Z=\calO(\nicefrac{1}{\sqrt{n}})$ through a chaining argument. Fix $\vw,\vw'\in\calW$. By Lemma~\ref{lem:lip}, note that $\vw\mapsto \sum_{\vy\in \calY(\vx)}p_\la(\vy|\featw(\vx))$ is Lipschitz with constant $L_\calW\sqrt{d(\vx)}/\lambda$. We obtain that for all $i\in[n]$, as $\frac{1}{n}\sum_{\vy\in \calY(X_i)}  p_\la(\vy|\featw(X_i))=1$,
\begin{align*}
    V_{i,\vw}-V_{i,\vw'}
        &= \frac{1}{n} \sum_{\vy\in \calY(X_i)}
        \big(p_\la(\vy|\featw(X_i))-p_\la(\vy|\psi_{\vw'}(X_i))\big)
        \fh(\vy,X_i)\\
        &=\frac{1}{n} \sum_{\vy\in \calY(X_i)}
        \big(p_\la(\vy|\featw(X_i))-p_\la(\vy|\psi_{\vw'}(X_i))\big)
        (\fh(\vy,X_i)-\inf_{\vx',\vy'}\fh(\vy',\vx'))\\
        &\le \frac{1}{n} L_{\calW}\frac{\sqrt{d(\mathcal X)}}{\lambda} \operatorname{osc}(\fh) \|\vw-\vw'\|=:a_n\,,
\end{align*}
and similarly $V_{i,\vw'}-V_{i,\vw}\le a_n$. For $S_\vw:=\sum_{i\in[n]} V_{i,\vw}$, we deduce that the random variable $(S_\vw-S_{\vw'})$ is sub-Gaussian as a sum of zero-mean independent sub-Gaussians \citep[Proposition~2.6.1, Page 29]{Vershynin_2018}, and it holds that for the associated $\psi_2$-norm \citep[Example 2.5.8, Page 28]{Vershynin_2018},
\begin{equation*}
    \|S_\vw-S_{\vw'}\|_{\psi_2}
        \le \Big[\frac{1}{\sqrt{\ln{2}}} \sum_{i\in[n]} \|V_{i,\vw}-V_{i,\vw'}\|^2_{\psi_2}\Big]^{\nicefrac{1}{2}}
        \!\!\le \frac{a_n \sqrt{n}}{(\ln{2})^{\nicefrac{1}{4}}} =\frac{L_{\calW}\sqrt{d(\mathcal X)}}{(\ln{2})^{\nicefrac{1}{4}}\lambda\sqrt{n}}  \operatorname{osc}(\fh) \|\vw-\vw'\|\,,
\end{equation*}
with, for a sub-Gaussian $X$, the $\psi_2$-norm defined as $\|X\|_{\psi_2}:=\inf\{t \ge 0\, |\, \bbE e^{\nicefrac{X^2}{t^2}}\le 2\}$. In other words, we have shown that the zero-mean random variable $S_\vw$ has sub-Gaussian increments. We can thus apply Dudley's inequality \citep[Theorem 8.1.3, Page 188]{Vershynin_2018},
\begin{align*}
    \bbE V=\bbE[\sup_{\vw\in\calW}S_\vw]&\le \frac{L_{\calW}\sqrt{d(\mathcal X)}}{(\ln{2})^{\nicefrac{3}{4}}\lambda\sqrt{n}}  \operatorname{osc}(\fh)\, \calI_{\calW} \,.
\end{align*}
We proceed similarly with $\inf_{\vw\in\calW} \sum_{i\in[n]} V_{i,\vw}$. We obtain that, setting $t=\sqrt{ \frac{16 (\operatorname{osc}(\fh))^2}{n}\ln{\frac{4}{\delta}}}$ in~\eqref{eq:A3_BLM}, with probability $1-2\delta$,
\begin{equation*}
    \Delta_n=\max\big\{\sup_{\vw\in\calW} \sum_{i\in[n]} V_{i,\vw},-\inf_{\vw\in\calW} \sum_{i\in[n]} V_{i,\vw}\big\}\le \bbE V + t
\end{equation*}
Using the bound on $\bbE V$, this yields \eqref{eq:bound_A3}.
\end{proof}

\subsection{Optimization error bound in the Kernel-SoS case}
\label{sec:k-SoS_control}
The last part of the proof consists in bounding the optimization error when minimizing $\calR_{n,\la}$. To alleviate notation, we set
\begin{equation*}
    R_{n,\la}(\vw):=\calR_{n,\la}(h_{\vw}) \, ,
\end{equation*}
which we want to globally minimize. It is a hard problem since $ R_{n,\la}$ is typically non-convex. The key idea in k-SoS is to find an (approximate) Sum-of-Squares representation for $R_{n,\la}$ while minimizing $R_{n,\la}$. Fix $s> 1+d_\calW/2$ and let $\calH_\phi$ be the Sobolev space over $\calW$
of smoothness $s$. Denote by $k$ its reproducing kernel, i.e.\ $k(\cdot,\vw)\in \calH_\phi$ such that for all $g\in \calH_\phi, \vw \in \calW$, we have $g(\vw)=\scalidx{g}{k(\cdot,\vw)}{\calH_\phi}$. Let $\phi:\calW \to \calH_\phi$ be the kernel embedding, i.e.\ $\phi(\vw)=k(\cdot,\vw)$. Set $\tilde s=s-d_\calW/2\ge 1$. Let  $S^+(\calH_\phi)$ be the set of self-adjoint
positive semidefinite operators from $\calH_\phi$ to $\calH_\phi$. We consider an assumption inspired by \cite[Assumption 1.a]{rudiFindingGlobalMinima2020}.

\begin{assum}{k-SoS}
\label{A_k-SoS}
We assume that $\calW$ is the closure of a union of closed balls $\cup_{\vw \in \calS} \bbB(\vw,r) $ for some $r>0$ and a bounded set $\calS\subset \R^{d_\calW}$. Assume furthermore that $R_{n,\la}\in C^{s+3}(\calW,\R)$ and that $R_{n,\la}$ has a kernel Sum-of-Square representation, i.e.\ there exists $A_{n,\la}\in S^+(\calH_\phi)$ such that
\begin{equation}
    R_{n,\la}(\vw) - \min_{\vw' \in \calW}R_{n,\la}(\vw')= \scalidx{\phi(\vw)}{A_{n,\la} \phi(\vw)}{\calH_\phi}, \, \forall \vw \in \calW. \label{eq:k-SoS_rep}
\end{equation}
For the following, we fix the $A_{n,\la} $ with minimal trace satisfying \eqref{eq:k-SoS_rep}.
\end{assum}

The requirement on $\calW$ is to use scattering bounds and is quite mild for compact sets. We mostly ensure $R_{n,\la}\in C^{s+3}(\calW,\R)$ by regularizing $R_{n,0}$ through $\la>0$ and using a smooth feature map $\vw \mapsto\featw$, e.g.\ $\featw(\vx)=\Phi_x\vw$ with $\Phi_x\in\R^{d(x)\times {d_\calW}}$. By \citet[Theorem 3]{rudiFindingGlobalMinima2020} (see also \cite{marteau2024second}), the existence of $A_{n,\la}$ is for instance guaranteed if $ R_{n,\la}$ and its second order partial derivatives belong to $\calH_\phi$, and $ R_{n,\la}$ has a finite number of minimizers, all with positive Hessian. We will see an explicit example below.

Recall that a global optimization problem can always be cast as a convex one under the form
    \begin{maxi*}|s|
		{\substack{c \in \R,}}{c}{}{}
		\addConstraint{c}{\le  R_{n,\la}(\vw),\, \forall\, \vw\in \calW.}
    \end{maxi*}
We consider the kernel Sum-Of-Squares approximation of it, namely
\begin{minie}|s|
		{\substack{c \in \R,\\ A \in S^+(\calH_\phi)}}{-c + \la_\phi \tr(A)}{\label{eq:k-SoS_optim}}{}\notag
		\addConstraint{R_{n,\la}(\tilde{\vw}_{m}) - c}{= \scalidx{\phi(\tilde{\vw}_{m})}{A \phi(\tilde{\vw}_{m})}{\calH_\phi},\, \forall\, m\in [M] \label{eq:k-SoS_optim_cons}}
\end{minie}
where $\la_\phi>0$ is some parameter, and the $(\tilde{\vw}_{m})_{m\in[M]}$ are sampled uniformly at random in $\calW$.

\begin{remark}
    \label{rem:k-SoS}
    A few words are in order. We want a $c\in\R$ such that $\vw\mapsto R_{n,\la}(\vw)-c\ge 0$. On the other hand, for any $A \in S^+(\calH_\phi)$, the function $\vw\mapsto\scalidx{\phi(\vw)}{A \phi(\vw)}{\calH_\phi}$ is nonnegative, and under some assumptions on $\phi$ and $\calH_\phi$ \citep{marteauferey20nonparametric}, these functions are pointwise dense in the nonnegative continuous functions. So we look for one such $A$ for which the two functions are equal at least at the sampled points as in \eqref{eq:k-SoS_optim_cons}, and we penalize the extra variable $A$ by its trace. For \eqref{eq:k-SoS_optim}, the Newton-based algorithm detailed in \citet[Section~6]{rudiFindingGlobalMinima2020} outputs the infimum value estimate $\hat{R}=\hat c$, the corresponding $\hat A$ and a candidate $\hat \vw_{n,\la}=\sum_{m\in[M]} \alpha_m \tilde{\vw}_{m}$ where each $\alpha_m$ is the Lagrange multiplier of the equality constraint \eqref{eq:k-SoS_optim_cons} at $\tilde{\vw}_{m}$.
    Contrary to \citet{rudiFindingGlobalMinima2020}, an extra term $|R_{n,\la}(\hat \vw_{n,\la}) - \hat{R}|$ appears in our analysis. One expects that this error goes to zero as $M$ grows large, and we refer to it as the \emph{a posteriori error}.
\end{remark}

\begin{theo}[Theorem 6, \cite{rudiFindingGlobalMinima2020}]\label{thm:k-SoS_certificate} Fix $\delta\in(0,1)$. With the above notation, under Assumption~\ref{A_k-SoS}, there exist explicit constants $\bar M\in\bbN$ and $\bar C\in\R_+$ depending on $s,{d_\calW},\delta,r,\operatorname{diam}(\calW)$ such that for $M\ge \bar M$ and $\lambda_\phi \ge \bar C M^{-\tilde s/{d_\calW}}\left(\log \frac{M}{\delta}\right)^{\tilde s/{d_\calW}}$ we have, with probability at least $1-\delta$,
    \begin{equation}\label{eq:optim_certificate}
    |R_{n,\la}(\hat \vw_{n,\la})-R_{n,\la}( \vw_{n,\la})|\le R_{n,\la}(\hat \vw_{n,\la}) - \hat{R}+ \la_\phi\left(\tr\left(A_{n,\la}\right)+|R_{n,\la}|_{\calW,\lceil \tilde s\rceil}\right),
\end{equation}
depending on the Sobolev norm $|R_{n,\la}|_{\calW,\lceil \tilde s\rceil}=\sup_{|\alpha|=\lceil \tilde s\rceil}\sup_{\vw' \in \calW}|D^\alpha R_{n,\la}(\vw')|$.
\end{theo}
The estimate \eqref{eq:optim_certificate} has the advantage that if $s \gg {d_\calW}$ then the smoothness mitigates the curse of dimensionality. Indeed, to achieve an a priori error $\epsilon$ on the last term of \eqref{eq:optim_certificate}, which we denote $\la_\phi C_{n,\la}$, we can take $M\propto \epsilon^{-{d_\calW}/\tilde s}$ rather than the more usual, and much larger, $M\propto \epsilon^{-{d_\calW}}$. Note however that the constants~$\bar M$ and $\bar C$ can still depend exponentially in the dimension, their closed-form expressions can be found in \citet[Theorem 5]{rudiFindingGlobalMinima2020}. Furthermore the Sobolev norm $|R_{n,\la}|_{\calW,\lceil s-{d_\calW} / 2\rceil}$ and the cost of the Sum-of-Square representation~$\tr\left(A_{n,\la}\right)$ depend crucially on the noise parameter $\la$ since the latter plays a role on the regularity of~$R_{n,\la}$, and we expect both terms to diverge to $+\infty$ when $\la$ goes to $0$. On the other hand, as just said, $\la_\phi$ can be taken very small when the number of sampling points $M$ is large, so that it compensates the other terms. In other words the computational complexity, which is polynomial in $M$, of solving larger SDP problems on $A$ will offset the cases when $\la$ is taken small and the function $R_{n,\la}$ is less smooth. We face thus a classical tradeoff between having a better estimate of $R_{n,0}$ or doing fewer computations.

The example below showcases how a simple strategy for the embeddings, which was implemented in \cite{parmentierLearningSolveSingle2021}, does satisfy the technical assumptions of this Section.

\subsection*{The case of generalized linear embeddings} Assume $\btheta=\featw(\vx)=\Phi_x\vw$ with $\Phi_x\in\R^{d(x)\times {d_\calW}}$ and the added noise $Z(x)$ is Gaussian with zero mean, then we can provide more explicit estimates for $\tr\left(A_{n,\la}\right)$ and $|R_{n,\la}|_{\calW,\lceil s-{d_\calW} / 2\rceil}$.

Decompose $Z(x)=\Phi_x Z'(x)+Z''(x)$ with $Z'$ Gaussian over $\R^{d_\calW}$ with covariance $\Sigma_x$ and $Z''$ Gaussian with values in the complement $\operatorname{Im}(\Phi_x)^\perp$ to the range of $\Phi_x$. Then
\begin{equation}
     \notag
     \bbE_{Z(x)}\big[  \fh(\hat \vy(\featw(\vx)+\la Z(\vx)), \vx) \big]=\bbE_{Z'(x)}\bbE_{Z''(x)}\big[  \fh(\hat \vy(\Phi_x (\vw+Z'(x))+\la Z''(x)), \vx) \big]
\end{equation}
Let $\tilde R_{n,\la,x}=\bbE_{Z''(x)}\big[  \fh(\hat \vy(\Phi_x \vw+\la Z''(x)), \vx)\big]$. We then observe that the objective function is a sum of Gaussian convolutions
\begin{equation}
    \notag
     R_{n,\la}(\vw)=\frac1n\sum_{i=1}^n \tilde R_{n,\la,x_i}\star \calN_{\la \Sigma_{x_i}}(\vw)
\end{equation}
so $R_{n,\la}$ is smooth. Moreover the convolution commutes with the derivation, and $|g\star g'|_{\infty}\le |g|_{\infty} |g'|_{1}$ for any functions $g,g'$. Furthermore $|R_{n,\la,x_i}|_{\infty}\le |f^0|_{\infty} $, whence for $\tilde s=s-{d_\calW} / 2$,
\begin{align*}
        |R_{n,\la}|_{\calW,\lceil \tilde s\rceil}\le
        &\frac1n\sum_{i=1}^n\sup_{|\alpha|=\lceil \tilde s\rceil}\sup_{\btheta\in\Omega}|\underbrace{D^\alpha R_{n,\la,x_i}\star \calN_{\la \Sigma_{x_i}}(\vw)}_{=R_{n,\la,x_i}\star D^\alpha \calN_{\la \Sigma_{x_i}}}|\\
        &\le |f^0|_{\infty}\sup_{i\in[n],|\alpha|=\lceil \tilde s\rceil}|D^{\alpha}\calN_{\la \Sigma_{x_i}}|_{1}\\
        &= \mathcal O\big(\lambda^{-\tilde s}\big)\,.
\end{align*}

Defining $\calM_{\la,i}:=\sqrt{\calN_{\la \Sigma_{x_i}}}$, clearly $\calM_{\la,i}\in\calH_\phi$ and we can obtain a closed-form sum-of-squares expression of $R_{n,\la}$ for the chosen k-SoS feature map $\phi$, so that we justify the assumption~\eqref{eq:k-SoS_rep}. Indeed by the reproducing property, we have
\begin{align*}
    R_{n,\la}(\vw)
    & =\frac1n\sum_{i=1}^n\scalidx{\phi(\vw)}{\left[\int_{\R^{d_\calW}}\calM_{\la,i}(\vw',\cdot)\otimes \calM_{\la,i}(\vw',\cdot) R_{n,0}(\vw')d\vw'\right]\phi(\vw)}{H_\phi}\\
    &=: \scalidx{\phi(\vw)}{\tilde A_{n,\la}\phi(\vw)}{H_\phi}\,.
\end{align*}
and since $A_{n,\la}$ was chosen as the operator with minimal trace, we have
\[
    \tr(A_{n,\la})\le \tr(\tilde A_{n,\la})=\int_{\R^{d_\calW}} \scalidx{\calM_{\la,i}(\vw',\cdot)}{\calM_{\la,i}(\vw',\cdot)}{H_\phi} R_{n,0}(\vw')d\vw'\,.
\]
Since the RKHS associated with $\phi$ is the Sobolev space of order $s$, we have a characterization of the norm through the Fourier transform
\begin{equation*}
    \scalidx{\calM_{\la,i}(\vw',\cdot)}{\calM_{\la,i}(\vw',\cdot)}{H_\phi}=\int_{\R^{d_\calW}} (1+|\xi|^2)^{s} |\hat \calM_{\la,i}(\vw',\xi)|^2 d\xi = \calO(\la^{-\tilde s}).
\end{equation*}

\begin{continuance}{ex:stoVSP}
    The stochastic vehicle scheduling problem (StoVSP) is an instance of generalized linear embedding. As such, the previous calculation holds.
\end{continuance}

\bibliographystyle{plainnat}
\bibliography{mlMeetOr}

\pagebreak

\begin{center}
\textbf{\huge Supplementary Materials}
\end{center}
\setcounter{equation}{0}
\setcounter{figure}{0}
\setcounter{table}{0}
\setcounter{page}{1}
\makeatletter
\renewcommand{\theequation}{S\arabic{equation}}
\renewcommand{\thefigure}{S\arabic{figure}}
\renewcommand{\bibnumfmt}[1]{[S#1]}
\renewcommand{\citenumfont}[1]{S#1}

\appendix

\section{Technical lemmas}
\label{app:technical_lemmas}
\begin{lem}
    \label{lemma_9}
    Let $u\,:\,\mathds R^d\to\mathds R$ be an integrable bounded function, $Z$ a standard normal random vector of $\mathds R^d$, $\sigma>0$ a positive real number, and $G(\vw):=\bbE_Z[u(\vw+\sigma Z)]$. Then $\vw\mapsto G(\vw)$ is $\sqrt d\|u\|_\infty/\sigma$-Lipschitz.
\end{lem}

\begin{proof}
    Let $h$ be the density of $\sigma Z$. One has $G(\vw)=\int_{\bbR^d}u(\vw+\vv)h(\vv)\mathrm d\vv=\int_{\bbR^d}u(\vv)h(\vv-\vw)\mathrm d\vv$. By dominated convergence, it holds that
    \[
        \nabla G(\vw)=-\int_{\bbR^d}u(\vv)\nabla h(\vv-\vw)\mathrm d\vv=-\int_{\bbR^d}u(\vw+\vv)\nabla h(\vv)\mathrm d\vv\,.
    \]
    By Jensen's inequality and standard calculation with the Gaussian density, one has
    \[
        \|\nabla G(\vw)\|_2\leq \|u\|_\infty \int_{\bbR^d} \|\nabla h(\vv)\|_2\mathrm d\vv\leq \frac{\sqrt d\|u\|_\infty}{\sigma}\,,
    \]
    which gives the result.
\end{proof}

\bigskip
\begin{lem}
\label{lem:cont_risk}
    Under Assumptions~\ref{lip} and \ref{A3}, the following holds. When $\varepsilon_0>0$, the risk $\vw\mapsto\calR_{\varepsilon_0}(h_\vw)$ is continuous with respect to $\vw$.
\end{lem}

\begin{proof}
    Let $\vx\in\calX$ and $\vy,\btheta\in\bbR^{d(\vx)}$. Invoke Lemma~\ref{lemma_9} with $u(\btheta)=p_0(\vy|\btheta)\in[0,1]$ to get that $\btheta\mapsto\bbE_{Z}[\sum_{\vy\in\calY(\vx)} p_0(\vy|\btheta+{\varepsilon_0} Z(\vx))\, \fh(\vy,\vx)]$ is continuous. Now, note that $\vw\mapsto\featw(\vx)$ is $L_{\calW}$-Lipschitz by Assumption~\ref{lip}, hence $\vw\mapsto\bbE_{Z}[\sum_{\vy\in\calY(\vx)} p_0(\vy|\featw(\vx)+{\varepsilon_0} Z(\vx))\, \fh(\vy,\vx)]$ is continuous. Taking the expectation, by dominated convergence, we deduce the result.
\end{proof}

\bigskip

\begin{lem}
\label{lip_theta_intermediate}
Let $\lambda>0$, $\vx\in\calX$ and $\vy\in\calY(\vx)$. Then, the function $\btheta\in\bbR^{d(\vx)}\mapsto \sum_{\vy\in \calY(\vx)} p_\la(\vy|\btheta)$ is $\sqrt{d(\vx)}/\lambda$-Lipschitz, under Assumption~\ref{A3}.
\end{lem}

\begin{proof}
    Note that $\sum_{\vy\in \calY(\vx)}p_\la(\vy|\btheta)=\bbE_Z[u(\btheta+\lambda Z)]$ where $u(\cdot)=\sum_{\vy\in \calY(\vx)}p_0(\vy|\cdot)\in[0,1]$ and apply Lemma~\ref{lemma_9}.
\end{proof}

\bigskip

\begin{lem}
\label{lem:lip}
    The function $\vw\in \calW\mapsto \sum_{\vy\in \calY(\vx)} p_\la(\vy|\featw(\vx))\in[0,1]$ is $L_\calW\sqrt{d(\vx)}/\lambda$-Lipschitz, under Assumptions~\ref{lip} and \ref{A3}.
\end{lem}

\begin{proof}
    Let $\vx\in\calX$ and $\vy,\btheta\in\bbR^{d(\vx)}$. Invoke Lemma~\ref{lip_theta_intermediate} to get that $\btheta\mapsto p_\la(\vy|\btheta)$ is $\sqrt{d(\vx)}/\lambda$-Lipschitz. Now, observe that $\vw\mapsto\featw(\vx)$ is $L_\calW$-Lipschitz to conclude.
\end{proof}

\bigskip

\begin{lem}
\label{lem:chi_2}
    Let $Z$ be a chi-square distribution with $d$ degrees of freedom. Let $t>d$ then
    \[
    \mathbb P[Z>t]\leq \exp(d/2)\exp(-t/4).
    \]
\end{lem}
\begin{proof}
    From Markov’s inequality in its exponential form \citep{Boucheron_Lugosi_Massart_2013}, one has
    \begin{align*}
    \mathbb P[Z>t]&\leq\inf_{0<u<1/2}\{\exp(-ut)\mathbb E(\exp(uZ))\}\\
    &\leq \inf_{0<u<1/2}\{\exp(-ut)(1-2u)^{-d/2}\}.
    \end{align*}
    Take $u=(1/2)(1-d/t)$ to get
    \begin{align*}
    \mathbb P[Z>t]&\leq \exp\Big(-\frac{t}{2}\big(1-\frac{d}{t}\big)+\frac{d}{2}\log\big(\frac{t}{d}\big)\Big)\\
    &\leq \exp\Big(-\frac{d}{2}\big(\frac{t}{d}-1-\log\big(\frac{t}{d}\big)\big)\Big)
    \end{align*}
    Set $v={t}/{d}-1>0$ and note that
    \begin{align*}
        v-\log(1+v)=\int_{1}^{1+v}\frac{w-1}{w}\mathrm{d}w\geq\big(\frac{1}{1+v}\big)\int_{1}^{1+v}({w-1})\mathrm{d}w= \frac{v^2}{2+2v}= \frac{t}{2d}-1+\frac{d}{2t}\,,
    \end{align*}
    hence
    \[
    \mathbb P[Z>t]\leq\exp\Big(-\frac{t}{4}+\frac{d}{2}-\frac{d^2}{4t}\Big)
    \]
    which gives the result.
\end{proof}

\begin{lem}
\label{lem:bounded_density_convolution}
Let $d \ge 1$ be an integer. Let $f_\vw(X)$ be any random vector in $\bbR^d$ generated by an arbitrary mapping $f_\vw: \calX \to \bbR^d$ and instance distribution $X$, with probability law $\mu_\vw$. Let $\xi \sim \mathcal{N}(0, \sigma^2 \mathrm{Id}_d)$ be a Gaussian random vector independent of $X$, for some $\sigma > 0$.

Define the randomized output $\psi_\vw(X) = f_\vw(X) + \xi$. Then the law of $\psi_\vw(X)$ is absolutely continuous with respect to the Lebesgue measure, and its probability density function $g_\vw(\theta)$ is uniformly bounded such that
$$ \|g_\vw\|_\infty \le \frac{1}{(2\pi\sigma^2)^{d/2}} \,, $$
unconditionally for all $\vw \in \calW$, even if the distribution of $f_\vw(X)$ is singular or supported on a lower-dimensional manifold.
\end{lem}

\begin{proof}
Because $f_\vw(X)$ and $\xi$ are independent, the law of their sum $\psi_\vw(X) = f_\vw(X) + \xi$ is the convolution of the law of $f_\vw(X)$ (denoted by the probability measure $\mu_\vw$) and the law of the Gaussian perturbation $\xi$. Since $\xi$ is an isotropic Gaussian, its density with respect to the Lebesgue measure is given by
\[
    \varphi_\sigma(\vz) = \frac{1}{(2\pi\sigma^2)^{d/2}} \exp\Big(-\frac{\|\vz\|_2^2}{2\sigma^2}\Big) \,.
\]
\noindent
By the convolution formula, the probability density function $g_\vw$ of $\psi_\vw(X)$ at any point $\btheta \in \bbR^d$ exists and
\[
    g_\vw(\btheta) = \int_{\bbR^d} \varphi_\sigma(\btheta - \vu) \mathrm{d}\mu_\vw(\vu) \,.
\]

Observe that the Gaussian density attains its global maximum at the origin. Thus, for any $\vz \in \bbR^d$, it holds that $\varphi_\sigma(\vz) \le (2\pi\sigma^2)^{-d/2}$. We substitute this uniform upper bound into the integral to obtain
\[
    g_\vw(\btheta) \le \int_{\bbR^d} \frac{1}{(2\pi\sigma^2)^{d/2}} \mathrm{d}\mu_\vw(\vu) = \frac{1}{(2\pi\sigma^2)^{d/2}} \int_{\bbR^d} \mathrm{d}\mu_\vw(\vu)
    =\frac{1}{(2\pi\sigma^2)^{d/2}} \,,
\]
for all $\btheta \in \bbR^d$ and all $\vw \in \calW$.

This bound depends strictly on the variance of the injected noise $\sigma^2$ and the subspace dimension $d$, and holds completely independently of the deterministic mapping $f_\vw$ or the underlying distribution of $X$, concluding the proof.
\end{proof}

\section{Perturbation bias with remarks}
\begin{prop}
    \label{prop:Gauss_V}
    If Assumption~\ref{A3} holds, then, for all $\la,\varepsilon$ such that $1-\varepsilon>\lambda> \varepsilon\geq 0$ and all $\vw\in\calW$,
    \[
    q_\vw(\la)\leq \inf_{a\in(0,1)}
    \Bigg\{\bbP_{X,Z}\Big(
    \frac{\rho(\psi_{\vw}(X)+\varepsilon Z)}{\sqrt{d(X)}}< (\lambda+\varepsilon)^a \Big)
    +\bbE_X\Big[\exp\Big(\frac{d(\vx)}2\Big)\Big]\exp\Big(-\frac1{4(\lambda+\varepsilon)^{2(1-a)}}\Big)
    \Bigg\}\,.
    \]
    Furthermore, under the following
    \noindent
        \textbf{Assumption \eqref{A2}},
        \begin{equation}
        \label{A2}
        \tag{$\text{Orl.}$}
        \exists \tau>0\text{ s.t.  }
            \bbE_{X,Z}\bigg[
                    \exp\Big(
                        \big(\frac{\rho(\psi_{\vw}(X)+\varepsilon Z)}{\sqrt{d(X)}}\big)^{-\tau}
                        \Big)
                    \bigg]<\infty\,,
    \end{equation}
    it holds
    \[
    q_\vw(\la)\leq C \exp\big(-(2\lambda)^{-\frac{2\tau}{2+\tau}}\big) \,,
    \]
    for some constant $C>0$ which does not depend on $\la$.

    \noindent
    And under the following
    \noindent
        \textbf{Assumption \eqref{A2bis}},
        \begin{equation}
        \label{A2bis}
        \tag{$\text{Mom.}$}
        \exists \tau>0\text{ s.t.  }
            \bbE_{X,Z}\bigg[
                    \Big(
                        \frac{\rho(\psi_{\vw}(X)+\varepsilon Z)}{\sqrt{d(X)}}
                        \Big)^{-\tau}
                    \bigg]<\infty\,,
    \end{equation}
    it holds
    \[
    q_\vw(\la)\leq C\lambda^{\tau}\mathrm{polylog}(\la)\,,
    \]
    for some constant $C>0$ which does not depend on $\la$,
    where $\mathrm{polylog}(\la)$ is a polynomial logarithm term.
\end{prop}

\begin{proof}
    Invoke Lemma~\ref{lem:chi_2} to get that, for $\vx\in\calX$, for all $t\geq d(\vx)$,
    \[
        \bbP[\|Z(\vx)\|_2^2\geq t]\leq \exp(\frac{d(\vx)}2)\exp(-\frac t4)\,,
    \]
    for the Gaussian model (see Assumption~\ref{A3}). Hence, for all $\la\in(\varepsilon,1-\varepsilon)$, all $a\in(0,1)$, all $\vx$,
    \[
        \bbP\Big[
        \|Z(\vx)\|_2^2\geq
        \frac{d(\vx)}{(\lambda+\varepsilon)^{2(1-a)}}
        \Big]
        \leq \exp(\frac{d(\vx)}2)\exp\Big(-\frac{d(\vx)}{4(\lambda+\varepsilon)^{2(1-a)}}\,\Big)\,.
    \]
    Furthermore, note that by triangle inequality, the distance to any set is $1$-Lipschitz, so for all $\vx,\vz$,
    \[
    \frac{\rho(\psi_{\vw}(\vx))}{\lambda}
    \geq
    \frac{\rho(\psi_{\vw}(\vx)+\varepsilon \vz)}{\lambda}
    -\frac{\varepsilon}{\lambda}\|\vz\|_2\,.
    \]
    Hence it follows that,
    \begin{align*}
        q(\lambda)
        &=\bbE_{X}
            \mathbb P_Z\Big[\|Z\|_2\geq \frac{\rho(\psi_{\vw}(X))}{\lambda}\,\big|\,X\Big]
            \,,\\
        &
        \leq
        \bbE_{X}
            \mathbb P_Z\Big[\big(1+\frac{\varepsilon}{\lambda}\big)\|Z\|_2\geq \frac{\rho(\psi_{\vw}(X)+\varepsilon Z)}{\lambda}\,\big|\,X\Big]
        \,,\\
        &
        =
        \bbE_{X,Z}
            \Big[\mathbf{1}_{\big\{\|Z\|_2\geq \frac{\rho(\psi_{\vw}(X)+\varepsilon Z)}{\lambda+\varepsilon}\big\}}\Big]\,,
        \\
        &
        =        \bbE_{X,Z}
            \Big[\mathbf{1}_{\big\{\|Z\|_2\geq \frac{\rho(\psi_{\vw}(X)+\varepsilon Z)}{\lambda+\varepsilon}\big\}
            }
            \Big(
            \mathbf{1}_{\big\{\rho(\psi_{\vw}(X)+\varepsilon Z)< (\lambda+\varepsilon)^a \sqrt{d(X)}\big\}}
            +\mathbf{1}_{\big\{\rho(\psi_{\vw}(X)+\varepsilon Z)\geq (\lambda+\varepsilon)^a \sqrt{d(X)}\big\}}
            \Big)
            \Big]
        \,,\\
        & \leq \bbP_{X,Z}\Big[\rho(\psi_{\vw}(X)+\varepsilon Z)< (\lambda+\varepsilon)^a \sqrt{d(X)}\Big]
        \\ & \quad + \bbE_{X,Z}
        \Big[
            \mathbf{1}_{\big\{\|Z\|_2\geq \frac{\rho(\psi_{\vw}(X)+\varepsilon Z)}{\lambda+\varepsilon}\big\}}
            \mathbf{1}_{\big\{\rho(\psi_{\vw}(X)+\varepsilon Z)\geq (\lambda+\varepsilon)^a \sqrt{d(X)}\big\}}
        \Big]
        \,,\\
        & \leq \bbP_{X,Z}\Big[\rho(\psi_{\vw}(X)+\varepsilon Z)< (\lambda+\varepsilon)^a \sqrt{d(X)}\Big]
        + \bbE_X\mathbb P_Z\Big[\|Z\|_2^2\geq \frac{d(X)}{(\lambda+\varepsilon)^{2(1-a)}}\,\big|\,X\Big]\,,\\
        & \leq \bbP_{X,Z}\Big[\rho(\psi_{\vw}(X)+\varepsilon Z)< (\lambda+\varepsilon)^a \sqrt{d(X)}\Big]
        + \bbE_X\Big[\exp(\frac{d(X)}2)\exp\Big(-\frac{d(X)}{4(\lambda+\varepsilon)^{2(1-a)}}\Big)\Big]\,,\\
        & \leq \bbP_{X,Z}\Big[\rho(\psi_{\vw}(X)+\varepsilon Z)< (\lambda+\varepsilon)^a \sqrt{d(X)}\Big]
        + \bbE_X\Big[\exp(\frac{d(X)}2)\Big]\exp\Big(-\frac{1}{4(\lambda+\varepsilon)^{2(1-a)}}\Big)\,,
    \end{align*}
    the last inequality stemming from $d(X)\geq 1$ $\bbP_X$-almost surely.

\medskip

    Assume \eqref{A2}, by Markov's inequality it holds
    \[
    \bbP_{X,Z}\Big[\rho(\psi_{\vw}(X)+\varepsilon Z)< (\lambda+\varepsilon)^a \sqrt{d(X)}\Big]\leq C\exp\big(-(\lambda+\varepsilon)^{-\tau a}\big)
    \leq C\exp\big(-(2\lambda)^{-\tau a}\big)\,.
    \]
    Taking $a=2/(2+\tau)$ gives the result.

\medskip

    Assume \eqref{A2bis}, by Markov's inequality it holds
    \[
    \bbP_{X,Z}\Big[\rho(\psi_{\vw}(X)+\varepsilon Z)< (\lambda+\varepsilon)^a \sqrt{d(X)}\Big]\leq C \,(\lambda+\varepsilon)^{\tau a}\,.
    \]
    Taking $a=1-\log(-4\tau\log\la)/(-2\log(\la))$ gives the result (note that $a<1$, and $a>0$ for $\la$ sufficient small).
    \end{proof}

\begin{remark}[On the conditions \eqref{A2} and \eqref{A2bis}]
    These conditions describe the tail behavior of the random variable $T:={\sqrt{d(X)}}/{\rho(\psi_{\vw}(X)+\varepsilon Z)}$. Condition~\eqref{A2} is of Orlicz-norm type: it states that~$T$ belongs to the $\tau$-Orlicz space and, as such, it has $\tau$-exponentially light tail; see for instance \cite[Section~1.1]{chafai2012interactions} for a review on Orlicz spaces. Hence, the distribution of the random variable ${\rho(\psi_{\vw}(X)+\varepsilon Z)}/{\sqrt{d(X)}}=1/T$ is $\tau$-exponentially light at the edge support point zero, hence the random variable $1/T$ is close to zero when switching from one normal cone to another. Roughly speaking, Condition~\eqref{A2} means that the random direction $\psi_{\vw}(X)+\varepsilon Z$ rarely falls close to a normal-cone boundary, i.e., the distribution of $\psi_{\vw}(X)+\varepsilon Z$ places only a small mass near the decision boundaries (the borders of the normal cones).

    Condition~\eqref{A2bis} is a moment type assumption on $T$. It is weaker than Condition~\eqref{A2}. Indeed, by \cite[Theorem 1.1.5]{chafai2012interactions}, note that \eqref{A2} is equivalent to
    \[
    \forall p\geq\tau \,,\quad
        \bigg[\bbE_X
                    \Big(
                        \frac{\rho(\psi_{\vw}(X)+\varepsilon Z)}{\sqrt{d(X)}}
                        \Big)^{-p}
                    \bigg]^{\frac1p}
                \leq C p^{\tau}\,,
    \]
    for some positive constant $C>0$. Hence, Condition~\eqref{A2} implies ~\eqref{A2bis}. Condition~\eqref{A2bis} gives however a slower rate of convergence of $q(\la)$ than Condition~\eqref{A2}.

    Note that both convergence rates are decreasing in $\tau$ and one cannot have better rates than the sub-Gaussian type bound $\mathcal O_{\lambda\to0}\big(\exp(-(2\lambda)^{-2})\big)$ with this type of analysis.
\end{remark}

\section{Excess Risk Bound for bounded conditional push-forward laws}

\begin{theo}[Excess Risk Bound]
\label{thm:linear_excess_risk}
Suppose Assumptions~\ref{A_emb}, \ref{A3}, \ref{lip}, and \ref{A_k-SoS} hold.
Furthermore, assume that Property~\ref{A_UBD} holds.
Then, for any $\la > 0$, the learned policy $h_{\vw_{M,n,\la}}$ satisfies:
\begin{align*}
0 \leq \calR_{\la}(h_{\vw_{M,n,\la}}) - \bbE_{X}\big[\fh(\vy^0(X), X)\big] &\leq C \la + \mathcal{O}_{\mathds P}\Bigg[ \frac{1}{\la\sqrt{n}} + \bigg( \frac{1}{\la(M/\log M)^{\frac{1}{d_\calW}}} \bigg)^{s-\frac{d_\calW}{2}} \Bigg] \\
&\quad + \underbrace{\big( \calR_{n,\la}(h_{\hat{\vw}}) - \hat{R} \big)}_{\text{a posteriori gap}} + \calE_0(\calH) \,,
\end{align*}
where $\calE_0(\calH) = \inf_{\vw \in \calW} \calR_0(h_\vw) - \bbE_{X}\big[\fh(\vy^0(X), X)\big]$ is the model misspecification error, $\widehat{\vw} = \vw_{M,n,\la}$ is the parameter returned by the k-SoS algorithm, and $C > 0$ is the constant from Proposition~\ref{prop:linear_bias_bound}.
\end{theo}

\begin{proof}
Let $\widehat{\vw} = \vw_{M,n,\la}$ denote the parameter returned by the k-SoS algorithm. We first show that the unregularized risk $\vw \mapsto \calR_0(h_\vw)$ is continuous. Let $(\vw_k)_{k \ge 1}$ be a sequence converging to $\vw \in \calW$. By Assumption~\ref{lip}, we have the convergence $\psi_{\vw_k}(X) \to \psi_\vw(X)$ for all $X$.

By the bounded density assumption, the conditional push-forward law of $\psi_\vw(X)$ is absolutely continuous with respect to the Lebesgue measure. Thus, the probability of $\psi_\vw(X)$ falling exactly on the boundaries of any normal cone $\calF_\vy$ is strictly zero. Consequently, for almost every $X$, the vector $\psi_\vw(X)$ lies strictly in the open interior of some normal cone. Since the interior of a normal cone is an open set, the convergence $\psi_{\vw_k}(X) \to \psi_\vw(X)$ implies that for $k$ sufficiently large, $\psi_{\vw_k}(X)$ falls into the exact same cone as $\psi_\vw(X)$. Thus, the discrete choice locks in almost surely: $\hat\vy_X(\psi_{\vw_k}(X)) = \hat\vy_X(\psi_\vw(X))$. By the Continuous Mapping Theorem, $\fh(\hat\vy_X(\psi_{\vw_k}(X)), X) \to \fh(\hat\vy_X(\psi_\vw(X)), X)$ almost surely. Because the objective $|\fh|$ is uniformly bounded by Assumption~\ref{A_emb}, we can invoke the Dominated Convergence Theorem to pass the limit inside the expectation:
$$ \lim_{k \to \infty} \calR_0(h_{\vw_k}) = \lim_{k \to \infty} \bbE_X \big[\fh(\hat\vy_X(\psi_{\vw_k}(X)), X)\big] = \bbE_X \big[\fh(\hat\vy_X(\psi_\vw(X)), X)\big] = \calR_0(h_\vw) \,. $$
This proves $\calR_0$ is continuous. Since $\calW$ is a compact set, the infimum is attained at some global minimizer $\vw^\star \in \calW$.

\medskip

    We can now decompose the difference between the risk of our learned policy $\widehat{\vw}$ and the risk of the optimal parameter $\vw^\star$:
\begin{align*}
\calR_{\la}(h_{\hat{\vw}}) - \calR_0(h_{\vw^\star})
&= \big( \calR_{\la}(h_{\hat{\vw}}) - \calR_{n,\la}(h_{\hat{\vw}}) \big) \\
&\quad + \big( \calR_{n,\la}(h_{\hat{\vw}}) - \calR_{n,\la}(h_{\vw_{n,\la}}) \big) \\
&\quad + \big( \calR_{n,\la}(h_{\vw_{n,\la}}) - \calR_{n,\la}(h_{\vw^\star}) \big) \\
&\quad + \big( \calR_{n,\la}(h_{\vw^\star}) - \calR_{\la}(h_{\vw^\star}) \big) \\
&\quad + \big( \calR_{\la}(h_{\vw^\star}) - \calR_0(h_{\vw^\star}) \big) \,.
\end{align*}

We bound each term on the right-hand side:
\begin{itemize}
    \item \textbf{Statistical Estimation Error:} The first and fourth terms are bounded by the uniform deviation over $\calW$:
    $$ \big( \calR_{\la}(h_{\hat{\vw}}) - \calR_{n,\la}(h_{\hat{\vw}}) \big) + \big( \calR_{n,\la}(h_{\vw^\star}) - \calR_{\la}(h_{\vw^\star}) \big) \le 2 \sup_{\vw \in \calW} \big| \calR_{\la}(h_\vw) - \calR_{n,\la}(h_\vw) \big| \,. $$
    By Theorem~\ref{thm:A3_statistical}, this supremum is bounded by $\mathcal{O}_{\mathds P}\big(\frac{1}{\la\sqrt{n}}\big)$.

    \item \textbf{Optimization Error:} The second term encapsulates the a posteriori gap and the approximation error from the k-SoS procedure. By Theorem~\ref{thm:k-SoS_certificate} and the subsequent analysis, this is bounded by:
    $$ \calR_{n,\la}(h_{\hat{\vw}}) - \calR_{n,\la}(h_{\vw_{n,\la}}) \le \big( \calR_{n,\la}(h_{\hat{\vw}}) - \hat{R} \big) + \mathcal{O}_{\mathds P}\Bigg[ \bigg( \frac{1}{\la(M/\log M)^{\frac{1}{d_\calW}}} \bigg)^{s-\frac{d_\calW}{2}} \Bigg] \,. $$

    \item \textbf{Empirical Minimization:} The third term is non-positive since $\vw_{n,\la}$ is the exact global minimizer of the empirical regularized risk $\calR_{n,\la}$:
    $$ \calR_{n,\la}(h_{\vw_{n,\la}}) - \calR_{n,\la}(h_{\vw^\star}) \le 0 \,. $$

    \item \textbf{Perturbation Bias:} The fifth term is the perturbation bias. Under the bounded density assumption, Proposition~\ref{prop:linear_bias_bound} guarantees that uniformly for all $\vw \in \calW$:
    $$ \calR_{\la}(h_{\vw^\star}) - \calR_0(h_{\vw^\star}) \le \big| \calR_{\la}(h_{\vw^\star}) - \calR_0(h_{\vw^\star}) \big| \le C \la \,. $$
\end{itemize}

Summing these bounds, we obtain:
$$ \calR_{\la}(h_{\hat{\vw}}) - \calR_0(h_{\vw^\star}) \le C \la + \mathcal{O}_{\mathds P}\Bigg[ \frac{1}{\la\sqrt{n}} + \bigg( \frac{1}{\la(M/\log M)^{\frac{1}{d_\calW}}} \bigg)^{s-\frac{d_\calW}{2}} \Bigg] + \big( \calR_{n,\la}(h_{\hat{\vw}}) - \hat{R} \big) \,. $$

Finally, we add and subtract the population optimum $\bbE_{X}\big[\fh(\vy^0(X), X)\big]$ to introduce the misspecification error $\calE_0(\calH)$. Recognizing that $\calR_0(h_{\vw^\star}) = \inf_{\vw \in \calW} \calR_0(h_\vw)$, we get:
\begin{align*}
\calR_{\la}(h_{\hat{\vw}}) - \bbE_{X}\big[\fh(\vy^0(X), X)\big]
&= \Big( \calR_{\la}(h_{\hat{\vw}}) - \calR_0(h_{\vw^\star}) \Big) + \Big( \inf_{\vw \in \calW} \calR_0(h_\vw) - \bbE_{X}\big[\fh(\vy^0(X), X)\big] \Big) \\
&\le C \la + \mathcal{O}_{\mathds P}\Bigg[ \frac{1}{\la\sqrt{n}} + \bigg( \frac{1}{\la(M/\log M)^{\frac{1}{d_\calW}}} \bigg)^{s-\frac{d_\calW}{2}} \Bigg] + \big( \calR_{n,\la}(h_{\hat{\vw}}) - \hat{R} \big) + \calE_0(\calH) \,,
\end{align*}
which completes the proof.
\end{proof}

To complete our theoretical analysis of the perturbation bias, we formally demonstrate that the linear rate $\mathcal{O}(\la)$ established in Theorem~\ref{thm:linear_excess_risk} and Proposition~\ref{prop:linear_bias_bound} is sharp. The following proposition provides a lower bound by constructing a simple, standard configuration satisfying the bounded density condition where the perturbation bias is strictly $\Omega(\la)$.

\begin{prop}[Sharpness of the Linear Perturbation Bias]
\label{prop:sharpness_lower_bound}
There exists a combinatorial optimization problem satisfying Assumption~\ref{A_emb} a Gaussian perturbation satisfying Assumption~\ref{A3}, and a statistical model~$\psi_\vw$ satisfying Property~\ref{A_UBD}, such that the perturbation bias satisfies:
$$ \big| \calR_{\la}(h_\vw) - \calR_0(h_\vw) \big| = \Omega(\la) \,. $$
\end{prop}

\begin{proof}
Consider a simplified 1-dimensional binary choice problem where the feasible set is $\calY = \{0, 1\} \subset \bbR$. The normal fan of the convex hull $\mathcal{C} = [0, 1]$ consists of two cones: $\calF_0 = \bbR_-$ and $\calF_1 = \bbR_+$, separated by a single boundary point at the origin $H = \{0\}$.
Let the true objective function simply be $\fh(y, x) = y$.

Suppose the statistical model parametrizes the direction $\theta = \psi_\vw(X) \in \bbR$ such that its push-forward law is the uniform distribution over the interval $[0, 1]$. Its probability density function is $f_{\psi_\vw}(\theta) = 1$ on a compact support, which perfectly satisfies Property~\ref{A_UBD}.

For the unregularized risk, because $\theta$ is almost surely strictly positive, the linear oracle always selects $\hat{y}(\theta) = 1$. The unregularized risk is therefore exactly:
$$ \calR_0(h_\vw) = \bbE_\theta[\fh(\hat{y}(\theta))] = \int_0^1 1 \, \mathrm{d}\theta = 1 \,. $$

Now, consider the smoothed policy with a standard Gaussian perturbation $Z \sim \calN(0, 1)$. The oracle selects $\hat{y}(\theta + \la Z) = 1$ if and only if $\theta + \la Z > 0$, which is equivalent to $Z > -\theta/\la$. Let $\Phi$ denote the cumulative distribution function of the standard normal distribution.  The smoothed risk evaluates to:
$$ \calR_\la(h_\vw) = \bbE_{\theta, Z}\big[\fh(\hat{y}(\theta + \la Z))\big] = \int_0^1 \bbP\left(Z > -\frac{\theta}{\la}\right) \mathrm{d}\theta = \int_0^1 \Phi\left(\frac{\theta}{\la}\right) \mathrm{d}\theta \,. $$
\noindent
The perturbation bias is the difference between the two risks:
$$ \calR_0(h_\vw) - \calR_\la(h_\vw) = \int_0^1 \left( 1 - \Phi\left(\frac{\theta}{\la}\right) \right) \mathrm{d}\theta \,. $$
\noindent
Applying the change of variables $u = \theta/\la$ (so that $\mathrm{d}\theta = \la \, \mathrm{d}u$), we obtain:
$$ \calR_0(h_\vw) - \calR_\la(h_\vw) = \la \int_0^{1/\la} \big( 1 - \Phi(u) \big) \mathrm{d}u \,. $$
\noindent
As the perturbation scale goes to zero ($\la \to 0^+$), the upper limit of the integral approaches infinity. Using the identity $1 - \Phi(u) = \bbP(Z > u)$, the integral of the tail probabilities yields the expectation of the positive part of the standard normal:
$$ \lim_{\la \to 0^+} \int_0^{1/\la} \big( 1 - \Phi(u) \big) \mathrm{d}u = \int_0^\infty \bbP(Z > u) \mathrm{d}u = \bbE[Z \mathbf{1}_{\{Z > 0\}}] = \frac{1}{\sqrt{2\pi}} \,. $$

Consequently, as $\la \to 0$, the bias behaves asymptotically as:
$$ \big| \calR_{\la}(h_\vw) - \calR_0(h_\vw) \big| = \frac{1}{\sqrt{2\pi}} \la - o(\la) \,, $$
which strictly proves that the perturbation bias is $\Omega(\la)$, concluding the proof.
\end{proof}
\begin{prop}[Bounded Density via the Coarea Formula]
\label{prop:coarea_density}
Let the instance space $\calX \subset \bbR^p$ be a compact set, and let the random instance $X$ admit a probability density function $f_X$ with respect to the Lebesgue measure that is uniformly bounded, i.e., $\|f_X\|_\infty \le \beta$ for some $\beta > 0$. Let $\psi_\vw : \bbR^p \to \bbR^d$ (with $p \ge d$) represent a deterministic neural network with smooth (Lipschitz and almost everywhere continuously differentiable) activation functions. Suppose the following hold uniformly for all $\vw \in \calW$:
\begin{enumerate}
    \item \textbf{Non-degeneracy:} The generalized Jacobian determinant $|J\psi_\vw|(\vx) \coloneqq \sqrt{\det(D\psi_\vw(\vx) D\psi_\vw(\vx)^\top)}$ is uniformly lower-bounded on $\calX$ by a constant $\alpha > 0$.
    \item \textbf{Bounded Level Sets:} The $(p-d)$-dimensional Hausdorff measure $\calH^{p-d}$ of the level sets is uniformly bounded, meaning there exists a constant $V_{\mathrm{level}} < \infty$ such that $\sup_{\theta \in \bbR^d} \calH^{p-d}\big(\psi_\vw^{-1}(\theta) \cap \calX\big) \le V_{\mathrm{level}}$.
\end{enumerate}

Then, for all $\vw \in \calW$, the push-forward law of $\psi_\vw(X)$ satisfies Property~\ref{A_UBD}: its support is compact, it is absolutely continuous with respect to the $d$-dimensional Lebesgue measure, and its probability density function $g_\vw(\theta)$ is uniformly bounded by $C_{\psi} = \frac{\beta}{\alpha} V_{\mathrm{level}}$.
\end{prop}

\begin{proof}
Since $\psi_\vw$ is a continuous mapping (due to the smooth activations) and the instance space $\calX$ is compact, the image $K_\vw = \psi_\vw(\calX) \subset \bbR^d$ is compact. Furthermore,
because $\calW$ and $\calX$ are compact, their product $\calW \times \calX$ is compact. Since the neural network mapping $(\vw, \vx) \mapsto \psi_\vw(\vx)$ is jointly continuous, its total image $K = \bigcup_{\vw \in \calW} \psi_\vw(\calX)$ is compact, and therefore closed and uniformly bounded.
Thus, the compact support requirement of Case ${\rm (i)}$ is satisfied.

To establish the existence and boundedness of the density $g_\vw(\theta)$, we use the coarea formula. For any measurable set $A \subset \bbR^d$, the probability that $\psi_\vw(X)$ falls in $A$ is:
$$ \bbP_X(\psi_\vw(X) \in A) = \int_{\calX \cap \psi_\vw^{-1}(A)} f_X(\vx) \mathrm{d}\vx \,. $$

Since $\psi_\vw$ is Lipschitz, Rademacher's theorem ensures it is differentiable almost everywhere. By the coarea formula, for any integrable function $h : \calX \to \bbR$, it holds that:
$$ \int_{\calX} h(\vx) |J\psi_\vw|(\vx) \mathrm{d}\vx = \int_{\bbR^d} \left( \int_{\psi_\vw^{-1}(\theta) \cap \calX} h(\vx) \mathrm{d}\calH^{p-d}(\vx) \right) \mathrm{d}\theta \,. $$

We choose the test function $h(\vx) = \frac{f_X(\vx)}{|J\psi_\vw|(\vx)} \mathbf{1}_{A}(\psi_\vw(\vx))$, which is well-defined almost everywhere because $|J\psi_\vw|(\vx) \ge \alpha > 0$. Substituting $h(\vx)$ into the coarea formula yields:
$$ \int_{\calX \cap \psi_\vw^{-1}(A)} f_X(\vx) \mathrm{d}\vx = \int_A \left( \int_{\psi_\vw^{-1}(\theta) \cap \calX} \frac{f_X(\vx)}{|J\psi_\vw|(\vx)} \mathrm{d}\calH^{p-d}(\vx) \right) \mathrm{d}\theta \,. $$

The left-hand side is exactly $\bbP_X(\psi_\vw(X) \in A)$. Because this holds for any measurable set $A$, the push-forward law is absolutely continuous with respect to the Lebesgue measure, and the integrand of the outer integral on the right-hand side is exactly the probability density function $g_\vw(\theta)$:
$$ g_\vw(\theta) = \int_{\psi_\vw^{-1}(\theta) \cap \calX} \frac{f_X(\vx)}{|J\psi_\vw|(\vx)} \mathrm{d}\calH^{p-d}(\vx) \,. $$
We can now bound $g_\vw(\theta)$ uniformly. By hypothesis, $f_X(\vx) \le \beta$ and $|J\psi_\vw|(\vx) \ge \alpha > 0$. Thus:
$$ g_\vw(\theta) \le \int_{\psi_\vw^{-1}(\theta) \cap \calX} \frac{\beta}{\alpha} \mathrm{d}\calH^{p-d}(\vx) = \frac{\beta}{\alpha} \calH^{p-d}\big(\psi_\vw^{-1}(\theta) \cap \calX\big) \,. $$
Finally, applying the bounded level sets assumption $\calH^{p-d}\big(\psi_\vw^{-1}(\theta) \cap \calX\big) \le V_{\mathrm{level}}$, we obtain:
$$ g_\vw(\theta) \le \frac{\beta}{\alpha} V_{\mathrm{level}} \eqqcolon C_\psi \,. $$
Because this upper bound depends only on $\alpha$, $\beta$, and $V_{\mathrm{level}}$, and holds independently of $\theta$ and $\vw$, the density $g_\vw$ is uniformly bounded, completing the proof.
\end{proof}

\newpage

\renewcommand{\arraystretch}{1.1}
\setlength{\tabcolsep}{6pt}
\begin{longtable}{p{0.1\textwidth}p{0.6\textwidth}p{0.2\textwidth}}
\toprule
Notation & Description & Ref. \\
\midrule
\endfirsthead
\toprule
Notation & Description & Ref. \\
\midrule
\endhead
\midrule \multicolumn{3}{r}{(continued)}\\\bottomrule
\endfoot
\bottomrule
\endlastfoot
$\calX$ & Instance space; $X$ random instance & — \\
$\calY(\vx)$ & Finite feasible solution set for instance $\vx$ & \eqref{eq:hardProblem} \\
$d(\vx)$ & Embedding dimension of $\calY(\vx)$ & Assump.~\ref{A_emb} \\
$\fh(\vy,\vx)$ & Original (black-box) cost & \eqref{eq:hardProblem} \\
$h, h_\vw$ & (Parametric) policy & \eqref{eq:risk}, \eqref{eq:modelClass} \\
$\calH,\calH_\calW$ & Policy class & \eqref{eq:modelClass} \\
$\vw \in \calW$ & Model parameter (compact set) & \eqref{eq:modelClass} \\
$d_\calW$ & Parameter dimension & — \\
$\featw(\vx)$ & Learned direction / feature map & \eqref{eq:modelClass} \\
$\bftheta$ & Shorthand for $\psi_\vw(\vx)$ & \eqref{eq:ylinearProblem} \\
$\hat\vy_\vx(\bftheta)$ & Linear optimization oracle solution & \eqref{eq:ylinearProblem} \\
$\calR(h)$ & Population risk & \eqref{eq:risk} \\
$\calR_n(h)$ & Empirical risk & \eqref{eq:regretMinimization} \\
$p_0(\vy|\bftheta)$ & Non-smoothed (baseline) distribution & \eqref{eq:def_p_0} \\
$p_\la(\vy|\bftheta)$ & Perturbed (smoothed) distribution & \eqref{eq:def_p_lambda} \\
$\lambda>0$ & Smoothing / perturbation scale & \eqref{eq:def_pbm_reg} \\
$Z(\vx)$ & Isotropic perturbation field (Gaussian)& Assump.~\ref{A3} \\
$\calR_{n,\la}$ & Regularized empirical risk & \eqref{eq:def_pbm_reg} \\
$\calR_{\la}$ & Regularized population risk & \eqref{eq:reg_risk_solution} \\
$\rho(\bftheta)$ & Distance to normal fan boundary & \eqref{eq:def_rho} \\
$q_\vw(\lambda)$ & Fan-crossing probability & \eqref{eq:def_Vw} \\
$\mathrm{UW}_{\varepsilon_0}$ & Uniform weak moment property & Propr.~\ref{A_law} \\
$\mathrm{UBD}$ & Uniformly bounded density property & Propr.~\ref{A_UBD} \\
$\varepsilon_0$ & Base regularization level & Prop.~\ref{prop:UW_conv} \\
$L_{\calW}$ & Lipschitz constant of $\psi_\vw$ & Assump.~\ref{lip} \\
$U$ & Uniform auxiliary variable (def. $p_0$) & \eqref{eq:def_p_0} \\
$\calF_\vy$ & Normal cone (fan cell) of $\vy$ & Around \eqref{eq:def_p_0} \\
$h^*, \vw^*$ & Optimal policy / parameter & \eqref{eq:risk}, \eqref{eq:optimal_risk} \\
$\conv(\calY)$ & Convex hull of solutions & --- \\
\caption{Main notation.}
\label{tab:notations_intro}
\end{longtable}

\end{document}